%% file: main-icaps26.tex
\documentclass[letterpaper]{article} % DO NOT CHANGE THIS
\usepackage{aaai2026}  % DO NOT CHANGE THIS
\usepackage{times}  % DO NOT CHANGE THIS
\usepackage{helvet}  % DO NOT CHANGE THIS
\usepackage{courier}  % DO NOT CHANGE THIS
\usepackage[hyphens]{url}  % DO NOT CHANGE THIS
\usepackage{graphicx} % DO NOT CHANGE THIS
\usepackage{natbib}  % DO NOT CHANGE THIS AND DO NOT ADD ANY OPTIONS TO IT
\usepackage{caption} % DO NOT CHANGE THIS AND DO NOT ADD ANY OPTIONS TO IT
\usepackage{algorithm}
\usepackage{amsmath}
\usepackage{algpseudocode}

\usepackage{mathtools} % Defines bigtimes

\usepackage{newfloat}
\usepackage{listings}
\DeclareCaptionStyle{ruled}{labelfont=normalfont,labelsep=colon,strut=off} % DO NOT CHANGE THIS
\floatstyle{ruled}
\newfloat{listing}{tb}{lst}{}
\floatname{listing}{Listing}
\nocopyright

\renewcommand{\thefootnote}{\fnsymbol{footnote}}

\title{Planning as Goal Recognition: Deriving Heuristics from Intention Models (Extended Version)}
\author{
    Giacomo Rosa\textsuperscript{\rm 1,\footnotemark[1]}, Jean Honorio\textsuperscript{\rm 1,\footnotemark[1]}, Nir Lipovetzky\textsuperscript{\rm 1,\footnotemark[1]}, Sebastian Sardi\~na\textsuperscript{\rm 2,\footnotemark[4]}
}
\affiliations{
    \textsuperscript{\rm 1}The University of Melbourne, \textsuperscript{\rm 2}RMIT University \\
    Melbourne, Australia \\
    \footnotemark[1]\texttt{\{first.last\}@unimelb.edu.au}, \footnotemark[4]\texttt{sebastian.sardina@rmit.edu.au}
}

\usepackage{bibentry}
\usepackage{makecell}
\usepackage[percent]{overpic}
\usepackage{subcaption}

\usepackage{url}
\usepackage{amsmath}
\usepackage{amsthm}
\usepackage{color}
\usepackage[table]{xcolor}

\DeclareMathOperator*{\argmax}{arg\,max}

\input{macros}

\begin{document}
\maketitle
\begin{abstract}
Classical planning aims to find a sequence of actions, a plan, that maps a starting state into one of the goal states.
If a trajectory appears to be leading to the goal, should we prioritise exploring it?
Seminal work in goal recognition (GR) has defined GR in terms of a classical planning problem, adopting classical solvers and heuristics to recognise plans. We come full circle, and study the adoption and properties of GR-derived heuristics for seeking solutions to classical planning problems.
We propose a new divergence-based framework for assessing goal intention, which informs a new class of efficiently-computable heuristics.
As a proof of concept, we derive two such heuristics, and show that they can already yield improvements for top-scoring classical planners.
Our work provides foundational knowledge for understanding and deriving probabilistic intention-based heuristics for planning.
\end{abstract}

%     % \link{Code}{https://aaai.org/example/code}
%     % \link{Datasets}{https://aaai.org/example/datasets}
%     \link{Extended version}{https://arxiv.org/abs/2603.14824}

% \footnotetext[1]{Corresponding author: rosag@student.unimelb.edu.au}
\renewcommand{\thefootnote}{\arabic{footnote}}

%%%%%%%%%%%%%%%%%%%%%%%%%%%%%%%%%%%%%%%%%%%%%%%%%%%%%%%%%%%%%%%%%
\section{Introduction} \label{sec:intro}
%%%%%%%%%%%%%%%%%%%%%%%%%%%%%%%%%%%%%%%%%%%%%%%%%%%%%%%%%%%%%%%%%

We study the connection between goal recognition (GR) and classical planning  by characterizing planning heuristics as mechanisms for assessing their goal-intentionality,  thus doing a full loop from \citet{RamirezGeffner:IJCAI09,RamirezGeffner:AAAI10}'s seminal work on GR as planning.
Classical planning is the field of AI that seeks to find a sequence of actions, a \emph{plan}, that maps an initial state in a problem into a state that satisfies a specific goal condition. In classical planning, actions are deterministic, states are fully observable and represented through binary variables (\emph{facts}), and no other actions occur outside the plan (i.e., static environment).
A common strategy adopted by solvers is to perform a search over the state space, using heuristics which estimate the distance to the goal, such as the FF~\cite{bonet2001planning,hoffmann2001ff} and Landmark~\cite{porteous2001extraction} heuristics, to guide the search.
Other techniques have also proven to be effective in enhancing search efficiency. These include identifying \emph{helpful actions}~\cite{hoffmann2001ff}, which prioritise operators likely to contribute to goal achievement, leveraging \emph{novelty} measures~\cite{lipovetzky2012width,lipovetzky2017best}, which favor exploration of states exhibiting previously unseen combinations of features, and \emph{dominance pruning}~\cite{torralba2015simulation}, which prunes states that provably cannot yield better solutions than previously explored states.
% While these techniques do not estimate goal distance directly, they complement heuristic search by prioritizing states or nodes that aid the search through alternative mechanisms.
These techniques do not estimate goal distance directly; rather, they prioritize states or nodes that aid the search through alternative mechanisms.
% \\
% \\
% \\

The goal recognition task involves an observer inferring an acting agent’s goals or plans based on partial observations of its behaviour~\cite{Sukthankar.etal:BOOK14-PAIR}.\footnote{Other common terms are Plan Recognition (PR) or Intention Recognition (IR). While subtle differences exists among them, in this paper, we shall use these terms interchangeably.}
Traditional approaches rely on a predefined~\emph{plan library}, which encodes known plans for implicit goals, allowing recognition through matching observed actions to entries in the library~\cite{KautzAllen:AAAI86}. The \emph{goal recognition as planning} framework~\cite{RamirezGeffner:IJCAI09,RamirezGeffner:AAAI10} casts recognition as a planning problem itself: a (declarative) goal is considered more likely if the observed actions align with an optimal or near-optimal plan to achieve it.
Recent contributions have extended this paradigm to account for irrational behaviour of agents \cite{MastersSardina:AIJ21}, as well as adopting estimated measures and using information contained in the \emph{effects} of observed actions to recognise goals and plans \cite{pereira2017landmark,wilken2024fact}.

In light of these recent developments, we revisit search for resource-bounded agents that seek to perform an intelligent exploration.
Due to resource limitations, such agents bias the search towards some fragment of all possible traces~\cite{Pollack:AIJ92-IRMA,Bratman:CI88-IRMA}.
We analyze the intentionality of these traces with the lens of work in GR, under the intuition that  some traces are seen as observations that are \qq{more intended} towards the goal than others, and push our search algorithms to explore more intended paths.
This process is framed as a GR problem where, rather than estimating the intentionality of one observation towards multiple goals, we aim to assign and compare the intentionality of different observations towards a distinguished goal.

\nspar{Contributions}
Our primary objective is to establish a novel conceptual framework that views heuristics as judges of the intentionality of discovered trajectories in a planning problem.
We first present a plan-library model of GR for a resource-bounded agent in a planning domain. This allows us to define the goal-intentionality of observed traces, and study the properties of solvers that use this quantity to seek plans.
A core contribution that separates our work from prior GR approaches is a new formulation of intentionality in which the likelihood is derived from the divergence between information in the observation and that in a description of goal-intended trajectories. Our plan-library model and proposed heuristics are both shown to belong to this family of estimators.
Our theory then informs a new class of \emph{intention-based} heuristics for classical planning.
As a proof of concept, we present two such heuristics, which we show help improve the state of the art in classical planning benchmarks.
We tie our results to our framework, providing experimental evidence of theorised properties of our search paradigm.

%%%%%%%%%%%%%%%%%%%%%%%%%%%%%%%%%%%%%%%%%%%%%%%%%%%%
\section{Preliminaries} \label{section_preliminaries}
%%%%%%%%%%%%%%%%%%%%%%%%%%%%%%%%%%%%%%%%%%%%%%%%%%%%

The classical planning model is defined as $\Phi=\tuple{S, s_0, S_G, A, f}$, where $S$ is the discrete finite state space, $s_0 \in S$ is the initial state, $S_G \subseteq S$ is the set of goal states, $A$ is the set of (deterministic) actions, and $f: A \times S \mapsto S$ denotes the model partial transition function, with $f(a, s)$ denoting the next state $s' \in S$ after applying action $a \in A$ in state $s \in S$. When $f$ is undefined, the action is not applicable in the state. We write $A(s)$ to denote the set of actions applicable in state $s$, i.e., $A(s) = \set{a \in A \mid f(a, s) \text{ is defined}}$.
A solution to a classical planning model is given by a \defterm{plan}, a sequence of actions $\tuple{a_{0}, \ldots, a_{m}}$ that induces a state sequence $\tuple{s_0, \ldots, s_{m+1}}$
such that $a_{i} \in A(s_i)$, $s_{i+1} = f(a_i, s_i)$, and $s_{m+1} \in S_G$ for $i \in \set{0, \ldots, m}$.

A STRIPS\footnote{Stanford Research Institute Problem Solver}~\cite{fikes1971strips} problem is defined through tuple $\P=\tuple{F,A,I,G}$, where $F$ denotes the set of boolean variables, or fluents, $A$ is the set of actions $a$, $I \subseteq F$ is the set of atoms that fully describe the initial state, and $G \subseteq F$ is the partial assignment that describes goal states. We assume unit cost actions in this work.

\paragraph{Planning Model Notation.}
Given a Classical Planning problem $\P$, a \defterm{trajectory} denotes a sequence of alternating states and actions $\tuple{s_{k}, a_{k}, s_{k+1}, a_{k+1}, \ldots, s_{m}, a_{m}, s_{m+1}}$, where both the first and last elements are states, such that $s_{i} \in S$, $a_{i} \in A(s_{i})$, and $s_{{i+1}} = f(s_{i}, a_{i})$; where $k \leq i < m+1$.
Every trajectory induces two projections: an \defterm{a-trajectory}, which is the sequence of actions $\tuple{a_{k}, a_{k+1}, \ldots, a_{m}}$, and an \defterm{s-trajectory}, which is the sequence of states $\tuple{s_{k}, s_{k+1}, \ldots, s_{m+1}}$. We use $\fancy{T}(s_{i},\pi)$ to represent the s-trajectory induced by an a-trajectory $\pi$ applied from state $s_{i}$.
We use $\L(s,\pi)$ to denote the last state $s_{m+1} \in S$ in the s-trajectory $\fancy{T}(s, \pi)$.
For simplicity, when $s$ is $s_0$ (the initial state of $\P$ as per $I$), we just write $\fancy{T}(\pi)$ and $\L(\pi)$, resp.

We place two constraints on considered trajectories: \emph{1)} \defterm{acyclic}: no state may appear more than once in a trajectory; and \emph{2)} \defterm{non-goal-extending}: goal states can only appear as the last state of the trajectory.
These are reasonable assumptions, as any cycle is redundant and extending a-trajectories beyond a goal state is superfluous for finding plans towards a single goal.
Given $S_i, S_j \subseteq S$, we use $\Pi(S_i,S_j)$ to denote the set of \emph{acyclic} and \emph{non-goal-extending} a-trajectories
that can be applied to a state $s_i \in S_i$ to yield a valid s-trajectory that begins at state $s_i$ and ends at state $s_j \in S_j$.
A \emph{plan} is then an a-trajectory $\pi \in \Pi(\{s_0\}, S_G)$; in other words, an a-trajectory that, when applied to $s_0$, reaches a valid goal state $s_g \in S_G$.
Note that more than one sequence of actions (a-trajectories) may yield the same history of states (s-trajectories).
Similarly, we use \defterm{I-reachable a-trajectory} to refer to all a-trajectories $\pi \in \Pi(\{s_0\}, S)$.
We adopt the definition of an \defterm{observation sequence} from previous plan and goal recognition literature \cite{RamirezGeffner:IJCAI09,MastersSardina:AIJ21} as \emph{any} sequence of actions $\tuple{o_1, ..., o_m}$, with $o_i \in A$. An action sequence \defterm{satisfies} an observation sequence iff it \defterm{embeds} it, meaning that there is a monotonic function $g$ that maps each observation $o_i \in A$ to the index of an identical action in the action sequence such that $g(o_i) < g(o_j)$ for all $j>i$.
It follows from the above definitions that, given the set $\O$ of all possible observation sequences for problem $\P$, $\Pi(\{s_0\}, S_G) \subseteq \Pi(\{s_0\}, S) \subseteq \Pi(S,S) \subseteq \O$.

An a-trajectory $\pi'$ \defterm{contains} a-trajectory $\pi$, written $\pi \sqsubseteq \pi'$, iff there exist, possibly empty, sequences of actions $\alpha$ and $\beta$ such that $\pi' = \alpha
\cdot \pi \cdot \beta$.
This relation is reflexive, i.e., an a-trajectory contains itself. An a-trajectory $\pi'$ \defterm{extends} an a-trajectory $\pi$, written $\pi \sqsubseteq_{\mathrm{pfx}} \pi'$, iff there exists a, possibly empty, suffix $\beta$ such that
$\pi' = \pi \cdot \beta$.
We define the set of \defterm{maximal a-trajectories} $\M$ as the set of I-reachable a-trajectories that are not extended by any other I-reachable a-trajectory in
$\Pi(\{s_0\},S):\M = \set{\pi \in \Pi(\{s_0\},S) \mid \ \forall \pi' \in ( \Pi(\{s_0\},S) \setminus \set{\pi} ), \ \pi  \not \sqsubseteq_{\mathrm{pfx}} \pi'}$.
Thus, $\Pi(\{s_0\},S_G) \subseteq \M \subseteq \Pi(\{s_0\},S)$.
We also define the operator $\propto_{\text{rank}}$, which indicates that two quantities induce the same ranking (the ordering is preserved): $k \propto_{\text{rank}} l
\;:=\;
 k(x) < k(y) \iff l(x) < l(y).$

\section{Intention-Based Search} \label{section_goal_search_model}

Goal recognition problems assume two agents: an acting agent, which follows a (hidden) trajectory, and an observer agent, whose task is to infer the intention of the acting agent given an observation, a partial trace of the actor’s full trajectory. We adapt this idea to the search problem in planning by imagining a one-vs-all GR problem, where the actor is either directed or not directed towards a single goal, and with known initial state. The observer must determine whether the actor is following a plan that satisfies the problem’s goal given a partial trace, and can therefore be thought of as a heuristic that determines the goal intention of an observation. Given multiple observations in an open list, we can then rank them based on their assigned probability of being intended towards the goal, preferring higher ranked observations for expansion as a means of guiding the search.
%-
We begin by assuming that the observer may be resource bounded, and as such may not have knowledge of all maximal a-trajectories starting from $s_0$. Rather, it has prior knowledge of a non-empty subset $\hat{\M} \subseteq \M$ of sampled candidate maximal a-trajectories, which it uses to infer the actor’s goal-intention.
%-
The observer is also subject to beliefs regarding the behaviour of the acting agent, which are expressed by assigning a weight to every maximal a-trajectory, a measure of preference for that a-trajectory. For example, if the observer believes that the actor is rational, it will assume that it is more likely to follow optimal or near-optimal paths towards its objective, and as such assign greater weights to shorter maximal a-trajectories.
Alternatively, it may assign uniform weights if its belief is that the actor prioritizes all maximal a-trajectories equally.

\subsection{IRPL Model}

We provide the \defterm{I-Reachable Plan-Library} (IRPL) model, that only considers I-reachable a-trajectories in a planning problem as valid observation sequences, and derives probabilities relative to an implicit library of sampled maximal a-trajectories, and the subset of those a-trajectories that constitute plans. This allows us to illustrate the usefulness of adopting such probabilities as heuristic signals in a planning problem, under the simplified scenario where probabilistic events are explicitly observable.
Given problem description $\tuple{F, A, I, G}$ and the set of all I-reachable a-trajectories $\Pi(\{s_0\},S)$,
the observer samples a set $\hat{\M}$ of known maximal a-trajectories starting at the initial state $s_0$, and assigns a weight to each maximal a-trajectory according to a \defterm{weight function} \( w : \M \to \mathbb{R}^{\scriptscriptstyle +} \).
Let the set of sampled plans be $\hat{\M}_G = \Pi(\{s_0\}, S_G) \cap \hat{\M}$. For I-reachable a-trajectories $O \in \Pi(\{s_0\},S)$, let $C(O)= \set{ \pi' \in \hat{\M} \mid O \sqsubseteq_{\mathrm{pfx}} \pi'}$ be the set of all sampled maximal a-trajectories that extend $O$.
Let $C_G(O) = \set{\pi' \in \hat{\M}_G \mid O \sqsubseteq_{\mathrm{pfx}} \pi'}$ be the set of sampled plans that extend $O$, and  $C_{\neg G}(O) = \set{\pi' \in (\hat{\M} \setminus \hat{\M}_G) \mid O \sqsubseteq_{\mathrm{pfx}} \pi'}$ be the sampled maximal non-plans extending $O$, such that $C_G(O) \bigcup C_{\neg G}(O)=C(O)$ and $C_G(O) \bigcap C_{\neg G}(O)=\emptyset$.

We define sets of maximal a-trajectories $\E \subseteq \hat{\M}$ that imply underlying events of interest, with probability

\begin{equation}
P(\E) = \sum_{\pi' \in \E} w(\pi') / \sum_{\pi'' \in \hat{\M}}w(\pi''). \nonumber
\end{equation}

\noindent Thus, $P(G)$ is the event that $\pi \in \hat{\M}$ is a plan:

\begin{equation}
P(G) := P(\hat{\M}_G) = \sum_{\pi' \in \hat{\M}_G} w(\pi') / \sum_{\pi'' \in \hat{\M}}w(\pi'').
\end{equation}\label{eq_p(G)}

\noindent Similarly, $P(\neg G):=P(\hat{\M} \setminus \hat{\M}_G)$.
$P(O)$ is the event that $\pi \in \hat{\M}$ extends $O$:
\begin{align}\label{eq_p(pi)}
    P(O) := P(C(O)) = \sum_{\pi' \in C(O)} w(\pi') / \sum_{\pi'' \in \hat{\M}}w(\pi'').
\end{align}

\noindent We can then obtain conditional probability $P(O \mid G) := P(C(O) \mid \hat{\M}_G)$:
\begin{equation} \label{eq_p(pi,G)}
    P(O \mid G) = \sum_{\pi' \in C_G(O)} w(\pi') / \sum_{\pi'' \in \hat{\M}_G}w(\pi''),
\end{equation}
\noindent
\noindent where $C_G(O) = C(O) \bigcap \hat{\M}_G$.\footnote{
    These probabilities are well-defined: $P(G)+P(\neg G)=1$, and $P(O)+P(\neg O)=1$, where $P(\neg O):=P(\hat{\M} \setminus C(O))$; calculating $P(O,G)=P(O \mid G) \cdot P(G)$ and using $P(\neg G)$ to obtain $P(O,\neg G)$, then $P(O,G)+P(O,\neg G)=P(O)$.
}
\noindent Finally, Bayesian posterior $P(G \mid O)$ becomes the weight of all sampled plans to the goal extending $O$, over the weight of all sampled maximal a-trajectories extending $O$:
\begin{equation} \label{eq_p(G|pi)}
    P(G\mid O) = \frac{P(O\mid G)\cdot P(G)}{P(O)} = \frac{ \sum_{\pi'' \in C_G(O)} w(\pi'') } { \sum_{\pi' \in C(O)} w(\pi')}.
\end{equation}

Any set of a-trajectories for which probabilities $P(O \mid G)$ and $P(G \mid O)$ are defined for all elements in the set can then be ranked to favour trajectories that are more goal-intended according to either likelihoods or posteriors, respectively,
\begin{equation} \label{argmax_1}
    \argmax_O P(O \mid G) = \argmax_O \sum_{\pi' \in C_G(O)} w(\pi'),
\end{equation}
\begin{equation} \label{argmax_2}
    \argmax_O P(G \mid O) = \argmax_O \frac{ P(O \mid G) } { P(O \mid \neg G) },
\end{equation}
\noindent where equation~\ref{argmax_2} is obtained by simplifying $\argmax_O P(G \mid O)/P(\neg G \mid O)$ and noting that $P(G)$ and $P(\neg G)$ are constant when considering a single goal in planning problems.
When extending the domain of conditional probabilities to the set of all possible observations in a planning problem, we adopt the convention of setting undefined probabilities to 0.
This reflects an observer that assumes unknown trajectories are not goal directed.

\subsection{Framework Properties}\label{section:framework_properties}
We study the properties of a search guided by Equations~\ref{argmax_1} and \ref{argmax_2}. We first state the results, followed by analysis.

\begin{claim} \label{obs_p_1}
    Given non-empty $\hat{\M}$ and $\hat{\M}_G$, any $w$, and a-trajectories $O_e$ extending an a-trajectory $O_p$ by one action, $P(O_e \mid G) \leq P(O_p \mid G)$, as $C_G(O_e) \subseteq  C_G(O_p)$.
\end{claim}

\begin{claim} \label{obs_p_2}
    $P(O \mid G)=0$ and $P(G \mid O)=0$ for all a-trajectories $O$ that are not extended by any plan $\pi' \in \hat{\M}_G$.
\end{claim}

\begin{lemma} \label{lemma_gen_1}
    Given non-empty $\hat{\M}$ and $\hat{\M}_G$, and any $w$, $\max_{O_e} P(G \mid O_e) \geq P(G \mid O_p)$.
\end{lemma}
\begin{proof}[Proof sketch]
    Every maximal a-trajectory extending $O_p$ extends exactly one child $O_e$; hence
    $C(O_p)=\bigcup_{O_e} C(O_e)$ and $C_G(O_p)=\bigcup_{O_e} C_G(O_e)$, and these
    unions are disjoint. Let $X(O)=\sum_{\pi' \in C(O)} w(\pi')$ and define $X_G(O)$
    analogously. Then $X(O_p)=\sum_{O_e} X(O_e)$ and $X_G(O_p)=\sum_{O_e} X_G(O_e)$, so
    \[
    P(G\mid O_p)=\frac{X_G(O_p)}{X(O_p)}
    =\sum_{O_e}\frac{X(O_e)}{X(O_p)}\frac{X_G(O_e)}{X(O_e)}.
    \]
    Thus $P(G\mid O_p)$ is a weighted average of its children and therefore cannot
    exceed all of them.
\end{proof}
\begin{theorem} \label{theorem_g_pi_expansions}
    Given non-empty $\hat{\M}$ and $\hat{\M}_G$, and any $w$, a planner that expands $\max P(G \mid O)$ and breaks ties by greater trajectory length, will find a  plan in number of expansions $m \leq \max_{\pi' \in \hat{\M}_G}\card{\pi'}$.
\end{theorem}
\begin{proof}
    The first expanded node has $P(G\mid O)$ greater than or equal to all others. Since $P(G\mid O)>0$,
    Claim~\ref{obs_p_2} ensures that at least one plan to the goal exists, and by
    Lemma~\ref{lemma_gen_1} its best child has probability $\ge$ that of its parent. Since it also has greater trajectory length, it is
    expanded next; induction completes the argument.
\end{proof}

\begin{lemma} \label{lemma_gen_2}
Let $\card{\hat{\mathcal{M}}_G} > 0$. Suppose $w$ is a weight function such that $w(\pi) > w(\pi') \iff \text{cost}(\pi) < \text{cost}(\pi')$.
For a planner that expands a-trajectories in order of $\max P(O \mid G) $, the first expanded goal node is guaranteed to be a minimal-cost plan among all plans in the sample \( \hat{\mathcal{M}} \).
The result also holds under the weaker condition $\text{cost}(\pi) \leq \text{cost}(\pi') \Rightarrow w(\pi) \geq w(\pi')$,
provided ties in $P(O \mid G)$ are broken by preferring shorter trajectories.
\end{lemma}
\begin{proof}[Proof sketch]
    Assume the first goal-reaching a-trajectory $\pi$ expanded is not minimal cost, and let
    $\pi^*$ be a cheaper plan. If $\pi^*$ is fully generated, then
    $P(\pi^*\mid G)>P(\pi\mid G)$ since $w(\pi^*)>w(\pi)$, contradicting that $\pi$
    was expanded first. If instead only its prefix $\rho$ is generated, then
    $P(\rho\mid G)\ge P(\pi^*\mid G)$ because prefixes aggregate the weight of all
    their extensions and these values decrease monotonically with depth. Thus
    $\rho$ (or $\pi^*$ when complete) would have been expanded before $\pi$.
    Hence the first expanded goal-achieving trajectory must correspond to a
    minimal-cost plan. When different-cost plans have equal weight, tie-breaking
    by shorter length selects the minimal-cost one first.
\end{proof}

\begin{lemma} \label{lemma_gen_4}
    A planner that expands a-trajectories in order of $\max P(O \mid G)$ will expand at most
    $\sum_{\pi \in \hat{\M}_G}({\card{\pi}}-1) + 1$
    nodes before expanding a goal, and $\sum_{\pi \in \hat{\M}_G}({\card{\pi}}-2) + 1$ nodes before generating a goal.
\end{lemma}
\begin{proof}
    Follows from Claim~\ref{obs_p_2} that at most all non-goal nodes in s-trajectories implied by plans in $\hat{\M}_G$ will be expanded before expanding a goal node. In the worst case, plans in $\hat{\M}_G$ do not overlap and only share the initial state, which is counted once.
\end{proof}

\begin{theorem}
\label{theorem_pi_g_num_samples}
    Given non-empty $\hat{\M}$ and $\hat{\M}_G$. For a planner that expands a-trajectories in order of $\max P(O \mid G)$; as samples are added to $\hat{\M}_G$, the length of the first expanded plan is non-increasing, and
    the worst-case number of expansions is non-decreasing.
\end{theorem}
\begin{proof}
    Follows from previous Lemmas~\ref{lemma_gen_2} and \ref{lemma_gen_4}, and considering that as samples are added, the minimum plan cost in the set can only decrease.
\end{proof}

\paragraph{Remarks.}
We briefly summarise general properties derived from the presented theorems. Claim~\ref{obs_p_2} implies that following any a-trajectory with both $P(O \mid G) > 0$ and $P(G \mid O) > 0$ is a valid strategy for reaching a goal. Theorem~\ref{theorem_g_pi_expansions} shows that expanding nodes according to Equation~\ref{argmax_2} follows a hill climbing strategy
when ties are broken by larger $g$
and is strongly goal directed.
If ties are broken by smaller $g$, then it may perform local searches when ties are encountered, until it finds an exit to the plateau.
In contrast, Equation~\ref{argmax_1}
follows the maximum a posteriori path,
inducing an exploratory strategy akin to
% breadth first search
an A* search with a consistent heuristic,
as shorter a-trajectories tend to have higher $P(O \mid  G)$, noted in Claim~\ref{obs_p_1}. This approach expands sampled solution trajectories until it selects a sample optimal plan. Theorem~\ref{theorem_pi_g_num_samples} reflects the sampling exploration-exploitation trade-off for $P(O \mid G)$: increasing the number of samples in $\hat{\M}_G$ can improve solution quality, but also increases the worst case number of expansions, reflecting the larger exploratory effort required.

\subsection{Uniform Regimes}

In what follows, we introduce and analyze the properties of two basic weight functions. We consider these as the two general uniform
weighting processes, where we assign equal probability
to, respectively (1) every sampled maximal a-trajectory, (2) every action choice in state transitions.

We define a \defterm{Uniform Maximal a-trajectory Probability} (UMP) weight function as a weight function $w(\pi)=c$ where $c$ is a positive constant,
implying a uniform preference of the agent towards any sampled maximal a-trajectory in $\hat{\M}$.
Let us define quantities $N_T=\card{\hat{\M}}$, $N_G=\card{\hat{\M}_G}$, $N_C(O)=\card{C(O)}$, and $N_{CG}(O)=\card{C_G(O)}$.

\begin{corollary} \label{cor_umpp_1}
    Given non-empty $\hat{\M}$ and $\hat{\M}_G$, and a UMP weight function $w$,
    the probabilities obtained become $P(O)=\frac{N_C(O)}{N_T}$, $P(G)=\frac{N_G}{N_T}$, $P(O \mid G)=\frac{N_{CG}(O)}{N_G} \propto N_{CG}(O)$,
    $P(G \mid O)=\frac{p(O \mid G)p(G)}{p(O)}=\frac{N_{CG}(O)}{N_C(O)}$.
\end{corollary}
\begin{proof}
    Follows from Equations~\ref{eq_p(G)},\ref{eq_p(pi)},\ref{eq_p(pi,G)}, and \ref{eq_p(G|pi)}; setting weight function $w(\pi)=1$, then the value of each summation is equivalent to the number of elements in the relevant sets.
\end{proof}

Corollary~\ref{cor_umpp_1} shows that under a UMP weight function, $P(O \mid G)$ is proportional to the number of plans extending $O$, and ordering the open list by $\argmax P(O \mid G)$ favours such prefixes. Ordering by $\argmax P(G \mid O)$ favours prefixes with a higher ratio of plan completions to non-plan continuations. Both quantities can be seen as \textit{measures of robustness} of a partial solution, biasing the search towards directions with more valid outcomes.

A \defterm{Uniform Transition Probability} (UTP) weight function assigns to a maximal I-reachable a-trajectory
\(\pi = \tuple{a_0, \dots, a_{k-1}}\) the product of uniform action probabilities at each step, $w(\pi) = \prod_{i=0}^{k-1} \big[\card{\text{A}(s_i)}\big]^{-1}$,
where \(s_0\) is the initial state, \(s_{i+1} = f(a_i,s_i)\), and \(\text{A}(s_i)\) is the set of applicable actions at state \(s_i\). That is, at each step the agent selects an applicable action with uniform probability.

\begin{lemma} \label{lemma_utp_1}
    Given non-empty $\hat{\M}$ and $\hat{\M}_G$, a UTP weight function $w$, and an a-trajectory $O$ that is extended by single solution plan $\pi_s$, the number of nodes generated to find $\pi_s$
    by a planner that expands according to $\max P(O \mid G)$ is lower bounded by $-\ln{[P(O \mid G) \cdot P(G)]} \cdot e \propto -\ln[P(O \mid G)]$.
\end{lemma}
\begin{proof}
    The lower bound on nodes generated is given by the minimum possible number of nodes generated while following $\pi_s$ that achieves UTP weight $w(\pi_s)=P(O \mid G) \cdot P(G)$. For each expanded state $s_i \in \T(\pi_s)$, the number of generated nodes increases by $\card{\text{A}(s_i)}$,
    and the weight of the a-trajectory to $s_i$ is multiplied by $\frac{1}{\card{\text{A}(s_i)}}$. For $w(\pi_s)=\frac{1}{X}$ the minimum number of generated nodes is thus given by solving $\min \sum_{s_i \in \T(\pi_s)}(\card{\text{A}(s_i)})$ s.t. $\prod_{s_i \in \T(\pi_s)}{\card{\text{A}(s_i)}}=X$.
    A lower bound to the integer solution is achieved by solving the real version of the problem, which can be solved analytically through the AM-GM Inequality to yield $e \cdot \ln(X)$.
\end{proof}

\begin{theorem} \label{theorem_utp_1}
    Given non-empty $\hat{\M}$ and $\hat{\M}_G$, a UTP weight function $w$, and a planner that expands according to $\max P(O \mid G)$,
    a lower-bound number of node generations required to achieve any plan that extends $O$ is $-\ln[ P(O \mid G)\ \cdot P(G)] \cdot e$.
\end{theorem}
\begin{proof}
    $P(O \mid G)$ is equivalent to the sum of the weight of all plans extending $O$, over a common denominator. From Lemma~\ref{lemma_utp_1}, it follows that the number of node generations required to solve any plan increases inversely to the plan's weight.
    Thus, the minimum number of node generations occurs when a single solution plan extends $O$.
\end{proof}

Theorem~\ref{theorem_utp_1} shows that expanding nodes according to Equation~\ref{argmax_1} with a UTP weight function follows the a-trajectory that minimises the best-case number of node generations to find a plan in $\hat{\M}$. This strategy can be seen as optimistic in the face of uncertainty, where uncertainty refers to unexplored regions of the state space. It assumes the subgraph extending the selected trajectory has an ideal shape; as new information is revealed, this estimate may worsen, leading the search to prefer other sub-graphs. Such a bound cannot be obtained with a UMP weight function, as uniform weights are not tied to the number of node generations.

\section{Estimating Measures of Goal Intention}
The IRPL model helps us characterize intention-based heuristics and their search behaviour.
The prior knowledge we have assumed so far is, however, unrealistic, as we cannot expect to have access to a maximal trajectory library $\hat{\M}$ and plan library $\hat{\M}_G$. To obtain a practically relevant framework, we need to account for an observer that estimates goal intentionality using approximate measures.

Existing work in GR-as-planning defines the likelihood $P(O \mid G)$ through the use of cost estimates \cite{RamirezGeffner:AAAI10,MastersSardina:AIJ21}. It assumes that rational agents are more likely to prefer lower-cost plans to their selected goal, and consequently relies on the sub-optimality of plans to each goal that incorporate observed events to estimate goal intentionality.
We avoid following this direction, because adopting such formulations would
lead us back to using (heuristic) cost estimates to direct the search.
Instead, we propose a novel formulation for approximating and interpreting goal intentionality which provides: 1) a theoretical background for alternative directions for estimating heuristics in classical planning, and 2) a generalisation of the IRPL model in terms of a larger class of approximate divergence-based goal-intended models.

\subsection{General Model of Divergence-based Likelihood}

We define a divergence-based generalisation of goal-intended likelihood through
Equation~\ref{equation_divergence_general},
\begin{equation} \label{equation_divergence_general}
    \tilde{P}(O \mid G) := e^{-D_{\mathrm{KL}}(d_{O,G} \parallel d_G)},
\end{equation}
where $d_G$ is a distribution that estimates goal-intended outcome probabilities (intuitively capturing likelihoods of outcomes in $\M_G$, $P(\pi \mid G)$), and $d_{O,G}$ denotes
a specific posterior distribution over the same space obtained by
incorporating observed evidence (intuitively capturing updated likelihoods of outcomes given observation $O$, $P(\pi \mid O,G)$).

It is important to keep in mind that observations may only reveal outcomes for a subset of the domain variables.
Thus, to illustrate the probabilistic intuition underlying the adopted distributions,
we provide the following simplified example of rolling a dice 10 times, where the goal $G$ is a total sum of values $>30$. A complete outcome is the sequence $r = (r_1,\dots,r_{10})$ of the random variable $R = (R_1,\dots,R_{10})$, where each $r_i$ denotes the value of one roll, and the distribution $d_G := \rho(r \mid G)$ assigns probabilities to all such sequences. An observation reveals the values of a subset of these variables, e.g., $(R_1=3, R_2=5)$, and observations are non-exclusive: observing $(3,5)$ does not exclude observing $(3,5,4)$ as a longer trace. The observation instead belongs to the marginal domain induced by the variables it reveals, with the marginal distribution $\rho(r_1,r_2 \mid G)$ assigning probabilities to the exclusive outcomes of that sub-domain. The posterior distribution then incorporates observations: $d_{O,G} := \rho(r \mid R_1=3, R_2=5, G)$.

Furthermore, in the IRPL model observations are defined as i-reachable a-trajectories, whereas now the domain of $d_G$ defines the outcomes that determine the events in observations. To illustrate this, following the previous example, assume each roll outcome $r_i$ no longer represents the value of the dice roll, but rather whether the roll gives a number $>3$. Observations would then take the form $(R_1=true, R_2=false,...)$. While approximate, such events still contain useful information for estimating goal achievement.

A well-defined KL-divergence requires matching the domains of $d_{O,G}$ and $d_G$.
Since the notation $O$ does not specify the domain of the observation, we instead adopt $x \in \X$, $y \in \Y$, and $z \in \Z$ to denote outcomes explicitly associated with their domains.
Let $\X = \X_1 \times \cdots \times \X_v$ be the domain of complete
outcomes, with $x = (x_1,\ldots,x_v)$ denoting a full outcome.  For any
index sets $\J \subseteq K \subseteq \{1,\ldots,v\}$, define the marginal
domains $\Y = \bigtimes_{i\in\J} \X_i$ and
$\Z = \bigtimes_{i\in K} \X_i$, with tuples
$y = (x_i)_{i\in\J}$ and $z = (x_i)_{i\in K}$, $x_i \in \X_i$, denoting outcomes in $\Y$ and $\Z$.  We write $\rho(x)$ for
the distribution over complete outcomes, and $\rho(y)$ and $\rho(z)$ for the
corresponding marginals.
Let $y^* \in \Y$ be the discrete deterministic observation $O$. We add superscripts to specify the domain of $d_{O,G}$ and $d_G$.
    Let $d_G^\X:=\rho(x \mid G)$ and $d_{O,G}^\Y:=\rho(y \mid y^*, G)$.
    Let $d_G^\Z$ and $d_{O,G}^\Z$ denote any domain-matched
    representation of $d_G^\X$ and $d_{O,G}^\Y$
    on an intermediate domain $\Z$.
    $d_G^\Z$ is marginalisation $\rho(z \mid G)=\frac{\rho(z \mid x, G)\cdot \rho(x \mid G)}{\rho(x \mid z, G)}$; $d_{O,G}^\Z$ is Bayesian update $\rho(z \mid y^*, G)=\frac{\rho(y^* \mid z, G)\cdot \rho(z \mid G)}{\rho(y^* \mid G)}$.
\begin{theorem}
\label{theorem_KL_invariance}
    The KL divergence is invariant to the choice of $\Z$, and coincides with the negative
    log-likelihood of the observation conditional on the goal:
    \[
        D_{\mathrm{KL}}(d_{O,G}^\Z \,\Vert\, d_G^\Z)
            \;=\; -\log \rho(y^* \mid G).
    \]
\end{theorem}
\begin{proof}
    Expanding $\rho(z \mid y^*, G)$, simplifying, and marginalising, we have
    \begin{align}
        &D_{\mathrm{KL}} ( d_{O,G}^{\Z} \parallel d_G^{\Z}) \nonumber
        =\sum_{z\in \Z}{\rho(z \mid y^*,G)\log{\frac{\rho(z \mid y^*,G)}{\rho(z\mid G)}}} \nonumber \\
        &=\sum_{z\in \Z}{\Big[\rho(z \mid y^*,G)\log{\rho(y^* \mid z,G)}\Big]} - \log{\rho(y^* \mid G)}. \nonumber
        \end{align}
        Since $y^*$ is a projection of some $z$ onto observed coordinates, for all $z\in \Z$
        if any variable value in $y^*$ is different from the corresponding variable in $z$, $\rho(z \mid y^*,G)=0$, and
        if all variables in $y^*$ have same value with corresponding variables in $z$, $\rho(y^* \mid z, G)=1$, so $\log{\rho(y^* \mid z, G)}=0$.
        Thus,
        \[
            D_{\mathrm{KL}} ( d_{O,G}^\Z \parallel d_G^\Z) = -\log{\rho(y^* \mid G)}. \qedhere
        \]
\end{proof}
Through Theorem 4, we thus show that the specific choice of intermediate domain used to align distributions, required for computing the KL-divergence, does not affect the derivation of $-\log \rho(y^* \mid G)$. This justifies Equation~\ref{equation_divergence_general}, and points at multiple possible strategies to estimate this quantity.

\nspar{Special cases} We briefly present cases of interest:

\noindent\emph{1) Marginalising $d_G^\X$ to $\Y$.} We can make $d_G^\X$ match the marginal distribution of $d_{O,G}^\Y$ by marginalising all domain factors not in $\Y$, directly obtaining likelihood $d_G^\Y:=\rho(y \mid G) = \frac{\rho(y \mid x,G) \cdot d_G^\X}{\rho(x \mid y, G)}=\sum_x{\rho(y,x \mid G)}$; $d_{O,G}^\Y$ is then a degenerate PMF $\rho(y \mid y^*, G)$, i.e., a one-hot distribution over its domain, assigning 1 to observation $y^*$, and 0 to all others.

\noindent\emph{2) Mapping $d_{O,G}$ to $\X$.} We can map the observation distribution to the full domain $\X$ through Bayesian update $d_{O}^{\X}:=\rho(x \mid y^*, G)=\frac{\rho(y^* \mid x,G) \cdot \rho(x \mid G)}{\rho(y^* \mid G)}$.

\noindent\emph{3) Consistency with IRPL Model.} When $d_G^\X$ perfectly describes the plan library, through a bijective mapping of every sampled plan to a possible outcome with its relative weight as probability, and $d_{O,G}^\Y$ describes an observed I-reachable a-trajectory, Equation~\ref{equation_divergence_general} then derives the IRPL likelihood. It thus generalises Equation~\ref{eq_p(pi,G)}. Details are included in Appendix~A.

\nspar{Assumptions}
The result holds whenever (i) the observation is discrete and
deterministic over $\Y$, i.e., $d_{O,G}^\Y(y)=\mathbf{1}[y=y^*]$, and (ii)
the observation map $m:\X \rightarrow \Y$ is surjective onto its image, ensuring that at
least one complete trace is consistent with $y^*$. These are consistent with the GR model.

\nspar{Interpretation}
The cost based formulation of goal recognition \cite{RamirezGeffner:AAAI10,MastersSardina:AIJ21} is grounded in the assumption of agent rationality: observations that align with near optimal plans are considered evidence of goal-intended behaviour.
Similarly, Equation~\ref{equation_divergence_general} interprets rationality in terms of the divergence of observed information $d_{O,G}$ from the distribution $d_G$, which captures statistical evidence of goal-intendedness from estimated solution paths.
It is intuitive that observations consistent with goal-intended behaviour,
when normalised by the likelihood of the observation,
reflect intent towards a goal.
In our experiments, we show that estimating such information can provide a signal for informing an efficient state-space exploration and traversal, whose behaviour still correlates with our results for the ideal but impractical IRPL model.

\section{Planning Heuristics}
We demonstrate the practical relevance of our conceptual framework by proposing two variants of a novel intention-based heuristic which align with our theory, allowing us to correlate results with our predictions of search behaviour obtained through analysis of the IRPL model.
We follow recent work in GR \cite{pereira2017landmark,wilken2024fact}, which represents observations through the facts implied by underlying action sequences or trajectories.
In a GR setting, this would represent facts added by actions in action sequence $O$. In a planning context, it may also include facts present in the initial state.

Our heuristic uses \defterm{fact observation probability estimation} as described in \cite{wilken2024fact} to estimate, for each fact $q \in F$, a Bernoulli
distribution $q \sim \mathrm{Bernoulli}(P_{\text{rel}}(q \mid G))$, representing the probability that $q$ is achieved at any point in a \emph{delete-relaxed} plan. $P_{\text{rel}}(q \mid G)$ then refers to the conditional probability in the relaxed problem. This is done by first sampling a set of delete-relaxed plans to the goal to obtain the probability of observing each action in a sampled relaxed plan.
The probability of observing facts is then indirectly derived as the probability of not observing any supporter action for the fact in a relaxed plan; thus, even if a fact is supported in all sampled relaxed plans, it may still get an observation probability less than 1.
Pseudocode for our implementation is included in Appendix~B.
\citet{wilken2024fact} adopt the technique in GR problems, using fact observation probabilities to weight vectors in a feature space, and then using the Euclidean norm of the vectors to estimate the distance covered by an agent towards a goal. Our divergence-based interpretation of goal-intendedness instead informs us to use fact observation probabilities to directly estimate the likelihood.

\nspar{Likelihood Estimation}
Let $O^F$ represent a set of observed facts. If we assume conditional independence between fact occurrences, the likelihood is simply
\begin{equation} \label{equation_p_x_g}
    \tilde{P}(O^F \mid G) = \prod_{q \in O^F}{P_{\text{rel}}(q \mid G)}. \nonumber
\end{equation}
According to our model, this likelihood can be interpreted in terms of a Boltzmann distribution of the marginalisation of $d_G:=\tilde\rho(q_1,...,q_{\card{F}})=\prod_{q\in F}\mathrm{Bernoulli}(P_{\text{rel}}(q \mid G))$ with respect to facts in $O^F$. It represents the divergence between distribution $\tilde\rho(q_1,...,q_{\card{F}})$ and the observation $O^F$, and is thus in the same class of divergence-based measures as the IRPL likelihood. This aspect justifies our comparison of solver behaviour with theory from our IRPL analysis.

In practice, we calculate the relative ranking of observations through the sum of log-probabilities. As the fact occurrence probability computation is only performed once
at the beginning of the search (at the initial state), $\tilde{P}(O^F \mid G)$ is then computable in time linear in $\card{O^F}$ during the search, providing quick evaluation.
In our experiments, we set the number of sampled relaxed plans to 100. This value was selected through analysis of results in the context of GR in \cite{wilken2024fact}, and empirical testing. We refer to this heuristic as \defterm{relaxed plan observation likelihood} (\defterm{ol$^{rp}$}).

\nspar{Posterior Estimation}
A Bayesian posterior is also derived from $P(G \mid O) \propto_{\text{rank}} \frac{P(O \mid G)}{P(O \mid \neg G)}$ (from Equation~\ref{argmax_2}):
\[
    \tilde{P}(G \mid O^F) \propto_{\text{rank}} \sum_{ q \in O^F } {\log P_{\text{rel}}(q \mid G) - \log P_{\text{rel}}(q \mid \neg G)}.
\]
This requires an additional estimation of $P_{\text{rel}}(q \mid \neg G)$, which uses a modification of the procedure used to obtain $P_{\text{rel}}(q \mid G)$: rather than sampling delete-relaxed plans, the heuristic samples delete-relaxed action sequences that achieve all the achievable non-goal facts in the instance.
Fact probabilities $P_{\text{rel}}(q \mid \neg G)$ are then extracted from these sets of supporter actions.
We still sample 100 fact sets each for both $P_{\text{rel}}(q \mid G)$ and $P_{\text{rel}}(q \mid \neg G)$, and call this variant \defterm{relaxed plan observation posterior} (\defterm{op$^{rp}$}).

\colorlet{Bluet1}{blue!8}
\begin{table*}[tbh]
    \centering
    \fontsize{8}{12}\selectfont
    \begin{tabular}{|c||c|c||c|c||c|c|}
    \hline
        & $f5$ & \cellcolor{green!15}$f5$-ol$^{rp}_r$ & $f5$ &  \cellcolor{green!15}$f5$-op$^{rp}_r$ & \cellcolor{green!15}$f5$-ol$^{rp}_r$ & \cellcolor{green!15}$f5$-op$^{rp}_r$ \\
        \hline
        Coverage (1831) & 1,510 & \cellcolor{Bluet1}1,560 (5.03) & 1,510 & \cellcolor{Bluet1}1,556 (1.5) & \cellcolor{Bluet1}1,560 (5.03) &  1,556 (1.5) \\
        \% Score & 76.77\% & \cellcolor{Bluet1}80.20\% (0.35) & 76.77\% & \cellcolor{Bluet1}79.90\% (0.15) & \cellcolor{Bluet1}80.20\% (0.35) & 79.90\% (0.15) \\
        N Fewer Expansions & 368.2 (6.7) & \cellcolor{Bluet1}1,134.4 (8.7) & 341.4 (9.1) & \cellcolor{Bluet1}1,159.4 (8.0) & 412.0 (7.3) & \cellcolor{Bluet1}716.6 (6.5)  \\
        N Lower Time & 582.4 (99.5)  & \cellcolor{Bluet1}897.0 (92.2) & 658.8 (117.3) & \cellcolor{Bluet1}814.6 (102.4)  & \cellcolor{Bluet1}811.8 (130.1) & 658.0 (122.6)  \\
        N Lower Plan Cost & \cellcolor{Bluet1} 760.6 (8.9) & 369.6 (7.6) & \cellcolor{Bluet1}763.6 (2.9) & 375.8 (1.7) & \cellcolor{Bluet1}375.2 (4.5) & 349.2 (13.4) \\
        \hline
        Average EpS & 40,111 & 35,680 & 40,111  & 33,626 & 35,680 & 33,626 \\
        \hline
    \end{tabular}
    \caption{Enhancing a basic BFWS($f5$) solver with proposed \textit{ol$^{rp}_r$} and \textit{op$^{rp}_r$} heuristics. \defterm{\% score} is the average of the \% of instances solved in each problem domain. Results for comparisons and solvers with a randomised component represent the mean, and include the standard deviation across 5 measurements. \textit{N Fewer Expansions, N Lower Time} and \textit{N Lower Plan Cost} represent the number of problems where a variant scores better in the respective metrics (ignoring ties).
    Avg. EpS represents the average number of expansions per second across problems solved by all planners, and proxies computational overhead.
    Results indicate that \textit{ol$^{rp}_r$} and \textit{op$^{rp}_r$} reduce the number of expansions across a significant portion of problems, on average improving also coverage and solution time. Such improvements come at the cost of on average worse plan quality compared to the baseline.
    }
    \label{tab:BFWS_1}
\end{table*}

\subsection{Adding Goal-intentionality to BFWS}
We integrate our proposed heuristics with a BFWS solver \cite{lipovetzky2017best}, which balances exploration of the search space and exploitation of heuristics, to evaluate improvements in Table~\ref{tab:BFWS_1}.
BFWS($f5$) uses \emph{Partitioned Novelty} \cite{lipovetzky2017best} as primary heuristic to partition each planning problem into multiple sub-problems and prioritize exploration, and the \emph{goal-count heuristic} $h^{GC}$, that counts the number of unachieved goal facts, is used both to inform such partitioning, and as tie-breaking heuristic. These heuristics are quickly computable, but often not very granular, creating a large number of ties. Minimum distance $g$ is used to break final ties, and has a big impact on the overall performance of the planner.
We implement our proposed variants as third tie-breakers after $h^{GC}$ to provide a fast but more informed tie-breaking mechanism that seeks progress by directing the search towards the goal. Our aim is for this to lead to states that improve the other heuristics more effectively than $g$. According to the IRPL model analysis, ol$^{rp}$ is supposed to exhibit an A*-like exploratory behaviour, augmenting the original tie-breaking mechanism, whereas op$^{rp}$ should induce a more greedy and goal-directed search.

\nspar{Partial Observation Trajectories}
We adapt $\text{ol}^{rp}$ and $\text{op}^{rp}$ to BFWS through \defterm{ol$^{rp}$-restart} (\defterm{ol$_r^{rp}$}) and \defterm{op$^{rp}$-restart} (\defterm{op$_r^{rp}$}). These adaptations add log probabilities from facts that have been observed only in the trajectory from the last state that improved $h^{GC}$, as opposed to all facts achieved from the start. By restarting at the most recent $h^{GC}$ improvement, it regains informedness in the cases where it was lost. Otherwise, if a fact that did not appear in any sampled relaxed plan to the goal is observed, it would strongly penalise the probability of all descendant nodes. This helps account for inaccuracies in fact occurrence estimates introduced by adopting distributions derived from relaxed plans.

\section{Experimental Results}
We run our experiments on an AMD EPYC 7763 processor, with each test running on a single core. We adopt Downward Lab's experiment module \cite{seipp2017downward}, whereas our proposed solvers and heuristics are implemented in C++ using the LAPKT planning module \cite{lapkt}. Our adopted branch of LAPKT uses the FD grounder \cite{helmert2009concise}, with the exception of problems that produce axioms, which are not currently supported in LAPKT. In such problems, our planners automatically fallback to the Tarski grounder \cite{tarski:github:18,singh2021grounding}. All experiments are limited to 1800 seconds time and 8 \textit{GB} memory constraints, following the satisficing track of the \textit{International Planning Competition} (IPC) \cite{ipc2023}. The problem set is composed of all IPC satisficing track benchmarks, selecting the latest problem sets for recurrent domains.

\nspar{Experimental Analysis}
Table~\ref{tab:BFWS_1} highlights the effect of incorporating proposed heuristics into BFWS($f5$). Both $f5$-ol$_r^{rp}$ and $f5$-op$_r^{rp}$ improve coverage and reduce the number of expansions on a significant portion of problems compared to the BFWS($f5$) baseline. This confirms that the goal-intendedness encoded by our heuristics induces a more efficient exploration of the state space. $f5\text{-op}_r^{rp}$ shows a small advantage over $f5\text{-ol}_r^{rp}$ in node expansions, aligning with our IRPL analysis that $P(G \mid O)$ drives a more goal directed greedy search. However, the more exploratory nature of $P(O \mid G)$ leads $f5\text{-ol}_r^{rp}$ to solve more instances overall. In regards to the exploration-exploitation dilemma, it is reasonable to expect that that the effectiveness of exploitation declines as the quality of the underlying estimates worsens. Thus, the more exploratory heuristic may be better suited to handle the approximation noise.
The additional overhead from computing fact observation probabilities is offset by improved search efficiency, resulting in lower average solution times than BFWS($f5$). By contrast, the extra computation of $P(O \mid \neg G)$ required by $\text{op}_r^{rp}$ introduces a small penalty relative to $\text{ol}_r^{rp}$, but differences are within the measurements' standard deviation.
The main drawback of both techniques is their effect on plan cost, which is expected due to the bias toward shorter plans induced by the original $g$ tie-breaker. Suboptimal plan costs may also be influenced by the trajectory-restart policy of $\text{ol}_r^{rp}$, which loses the global negative bias towards shorter trajectories from Claim~\ref{obs_p_1}.

\nspar{Enhanced Variants}
We also evaluate enhanced versions of the proposed solvers from Table~\ref{tab:BFWS_1} in Table~\ref{tab:BFWS_2}.
\defterm{BFWS($f5$)-Landmarks-\textit{ol}$_r^{rp}$} (\defterm{OL$_r^{rp}$}) substitutes $h^{GC}$ with the \defterm{Landmarks heuristic}
\cite{porteous2001extraction}, and uses a \defterm{single trimmed open list} \cite{rosa2024count} for greater memory efficiency. Both the coverage and average solution time gap between $\text{ol}_r^{rp}$ and $\text{op}_r^{rp}$ widen in this configuration (Appendix~D), hence we do not include the latter.
We also test variant \defterm{OL$_r^{rp}$-UTP} with modified fact observation probability sampling, to highlight the practical impact of the weighting scheme in the IRPL model.
When calculating supporter action probabilities for $\text{ol}_r^{rp}$, each sampled relaxed plan is given an equal weight. {OL$_r^{rp}$-UTP} re-weights sampled relaxed plans according to a UTP weight function, giving more importance to relaxed plans with greater UTP weight.
Lastly, we provide a comparison with recent high-coverage dual strategy solvers $-$ that run an initial search  and, if this fails, fall back to a backend solver $-$ by replacing the frontend solver in BFNoS-Dual \cite{rosa2024count} with OL$_r^{rp}$ (\defterm{OL$_r^{rp}$-Dual}).

\colorlet{Blue1}{blue!22}
\colorlet{Blue2}{blue!12}
\colorlet{Blue3}{blue!5}
\begin{table}[tbh]
    \centering
    \resizebox{\columnwidth}{!}{
    \setlength{\extrarowheight}{2pt}
    \begin{tabular}{c|ccc}
        \hline
         \textbf{Planner} & \textbf{Coverage (1831)} & \textbf{\% score }& \textbf{Agile score} \\
        \hline
        Dual-BFWS & 1,607 & 83.6\% & 1,200.8  \\
        ApxNoveltyT & 1,611 (3.5) & 83.8\% (0.2) & \cellcolor{Blue2}1,233.7 (0.2) \\
        LAMA & 1,535 & 79.1\% & 1,192.3 \\
        Scorpion-M & 1,591 & 82.9\% & 1,206.4  \\
        \cellcolor{green!15}OL$_r^{rp}$ & \cellcolor{Blue3}1,621 (3.2) & \cellcolor{Blue3}84.6\% (0.3) & 1,229.4 (3.4)  \\
        \cellcolor{green!15}OL$_r^{rp}$-UTP & 1,616 (2.1) & 84.2\% (0.1) & \cellcolor{Blue1}\textbf{1,236.0 (1.9)}  \\
        BFNoS-Dual & \cellcolor{Blue2}1,641 (0.6) & \cellcolor{Blue2}86.2\% (0.1) & 1,173.3 (3.5)  \\
        \cellcolor{green!15}OL$_r^{rp}$-Dual & \cellcolor{Blue1}\textbf{1,655 (1.5)} & \cellcolor{Blue1}\textbf{87.0\% (0.1)} &  \cellcolor{Blue3}1,232.6 (2.9)  \\
        \hline
    \end{tabular}
    }
    \caption{Mean coverage and Agile score of enhanced variants.
    Our enhancements of BFWS($f5)$ outperform SoTA planners, on average solving more problems, more quickly, without using multiple open lists or runs.
    OL$_r^{rp}$-Dual provides an improved dual-strategy solver, solving more problems than all other tested planners and attaining a meaningfully higher Agile score than BFNoS-Dual, the second-best in coverage.
    }
    \label{tab:BFWS_2}
\end{table}

\nspar{Results}
%%%%%%%%%%%%%%%%%%%%%%%%%%%%%%%%%%%%%%%%%%%%%%%%
Table~\ref{tab:BFWS_2} compares enhanced variants with multiple SoTA Novelty planners \emph{BFNoS-Dual} \cite{rosa2024count}, \emph{Dual-BFWS} \cite{lipovetzky2017best}, and \emph{Approximate Novelty Tarski} \cite{singh2021approximate}, as well as past IPC satisficing track winners \emph{LAMA} \cite{richter2010lama} and \emph{Scorpion-Maidu} \cite{correa-et-al-ipc2023c}. Results indicate improved coverage of our proposed modifications of BFWS($f5$). {OL$^{rp}_r$-UTP} also improves \defterm{Agile score}\footnote{\defterm{Agile score} is a performance metric that jointly evaluates coverage and runtime. It assigns a score of 1 for problems solved in time $T \leq 1s$, and $1-\frac{\log(T)}{\log(300)}$ for $1s<T \leq 300s$. The total score is summed across all problems.} over the base {OL$^{rp}_r$}, at the cost of some problem coverage.

We further note a correlation between our experimental results for {OL$^{rp}_r$-UTP} and Theorem~\ref{theorem_utp_1}, suggesting that prioritising sampled relaxed plans with tighter estimated bounds on the number of node generations to reach the goal can accelerate search, albeit at the expense of coverage, as it may reduce the robustness that comes with using uniform weights discussed in Corollary~\ref{cor_umpp_1}.
UTP weights promote a more committed search, whereby a focus on optimistic relaxed plan samples can lead to earlier solutions when these estimates are accurate, but may also mislead the search when they are not.
The latter case arises when sampled relaxed plans diverge from valid plans, potentially omitting important actions or facts, or including unhelpful ones.

As fact probability extraction relies on the sampling of relaxed plans, we expect the quality of approximations to align with that of well-studied relaxed heuristics. For instance, our planners perform strongly on Settlers \cite{long20033rd}, a challenging domain where relaxed-plan-heuristic planners LAMA and Maidu outperform BFWS baselines (Table~\ref{tab:comparative_performance_large2}). Moreover, sets of relatively poor plan-cost cases for variants in Table~\ref{tab:BFWS_1} appear to be domain-specific. We do not exclude that additional mechanisms related to the probability estimation process, such as observing sets of facts rather than actions, may impact per-domain amenability, informing future directions to improve proposed heuristics.

\section{Concluding Remarks}
We introduce a model that provides an intuitive description of goal-intendedness given a set of underlying plans in an instance, and use it to study the behaviour of planners that adopt goal intentionality as heuristic signal. Our model and proposed heuristics are shown to belong to a new class of divergence-based goal recognition estimates.
As a proof of concept, we propose two heuristics which improve the base performance of BFWS, matching complex IPC planners in Agile scores, while exceeding their coverage.
We provide empirical evidence of correlation between our experimental results and the properties of our simplified theoretical model.

Our proposed planning-as-goal-recognition theoretical framework offers a new perspective on heuristic search, interpreting evaluation functions as processes that infer the intention of trajectories to the current state with respect to the goal. Trajectories, therefore, do not only reveal the cost so far, but also their goal-intendedness.
This aspect can inform the design of new heuristics, and may provide a novel lens for analysing the properties of well-established classical planning heuristics.
The one-off time cost of the information extraction phase in the proposed heuristics opens the door to more informed and expensive estimation methods to further improve problem solving capability.
The probabilistic nature of the intentionality framework can also lead to new solutions in related problems, such as learned heuristics and hybrid planning, and facilitate connections between planning, goal recognition, and non-symbolic fields of AI.
% }

\section{Acknowledgements}

We thank the reviewers for their valuable feedback. Giacomo Rosa is supported by the Melbourne Research Scholarship established by the University of Melbourne. This research was supported by use of the Nectar Research Cloud and by the Melbourne Research Cloud. The Nectar Research Cloud is a collaborative Australian research platform supported by the NCRIS-funded Australian Research Data Commons (ARDC).

\bibliography{bib-keys-full,aaai2026.bib}

\clearpage
\twocolumn[
\vspace*{1em}
\begin{center}
{\LARGE \bf Technical Appendix \par}
\vspace{6em}
\end{center}
]

\label{AppA}
\section{Appendix A: Describing IRPL Likelihood through the Divergence-based Model.}

\paragraph{IRPL Likelihood}
In the IRPL model, the joint probability of a plan $\pi$ with observation $O$ is $P(O,\pi \mid G)=P(O \mid \pi,G)P(\pi \mid G)$, where $P(O \mid \pi,G)$ is 1 if $O \sqsubseteq_{\mathrm{pfx}} \pi$, and 0 otherwise. We can rearrange Equation~4 such that the denominator normalizes each weight in the numerator summation, to clearly visualize that adding probabilities for plans that extend $O$ is equivalent to marginalizing the joint probability:
\begin{align}
    P(O \mid G) &= \sum_{\pi' \in C_G(O)} \frac{w(\pi')}{\sum_{\pi'' \in \hat{\M}_G}w(\pi'')} \nonumber \\
     &= \sum_{\pi' \in C_G(O)} P(\pi' \mid G) \nonumber \\
    &= \sum_{\pi\in\hat{\M}_G}P(O, \pi \mid G) \nonumber
\end{align}

\vspace{4em}
\paragraph{Distributions $d_{O,G}$ and $d_G$}
For calculating the KL-divergence, we assume distribution $d_{O,G}$ perfectly describes the plan library, representing the exact domain and distribution of full plans i.e. each discrete variable is a time-step, with the discrete values representing the action selected at that time-step, and a dummy action for terminated plans (because they are non-goal-extending), to make them match the domain of the longest plan in the plan library. Let $\X_{\text{IRPL}}$ represent this domain. 

Every exclusive outcome $x \in \X_{\text{IRPL}}$ thus coincides with a plan $\pi \in \hat{\M}_G$ through a bijective mapping, with probability of $x$ equal to 
\begin{equation} \label{eq_app_1}
d_{O,G}(x):=\rho(x \mid G)=P(\pi_x \mid G)=\frac{w(\pi_x)}{\sum_{\pi' \in \hat{\M}_G}w(\pi')}
\end{equation}
where $\pi_x$ is used to represent the corresponding plan for $x$.

$y^*\in \Y$ similarly represents observed I-reachable a-trajectory $O$, with the domain $\Y$ being a marginal domain of $\X_{\text{IRPL}}$, as described in the section. For every plan $\pi \in C_G(O)$ there is a corresponding outcome $x_{\pi}$ that is consistent with $y^*$ and has probability $\rho(x \mid G)$; that is, 
\begin{equation} \label{eq_app_2}
    \rho(y^* \mid x, G)=P(O \mid \pi_x,G)
\end{equation}
$d_{O,G}^\Y$ then represents the one-hot distribution $\rho(y \mid y^*, G)$.

\paragraph{KL-Divergence Model Likelihood}
 Following discussed divergence-likelihood special case \#1, for outcomes $y$ in the domain of observation $y^*$, we have that 
 
 % we marginalize every variable in $d_{O,G}$ from a time-step greater than the last time-step of the observed I-reachable a-trajectory $y^*$, and select the exclusive event for observation $y^*$ from the marginal distribution $\rho(y \mid G)$. 
%  % $d_{O,G}^\Y$ becomes a marginalization of the joint probability 
% \[
% \sum_x\rho(y,x|G)=\rho(y \mid G)
% \]
% \[
%     D_{\text{KL}(d_{O,G}^\Y || d_{O,G}^\Y)=-\log{\rho(y^* \mid G)}}
% \]

\[
e^{-D_{\text{KL}}(d_{O,G}^\Y || d_{G}^\Y)}=\rho(y^* \mid G)
\]
From (\ref{eq_app_1}) and (\ref{eq_app_2}) we thus have that
\begin{align}
\rho(y^*\mid G)&=\sum_x\rho(y^*,x \mid G) \nonumber \\
&=\sum_x\rho(y^* \mid x,G)\rho(x \mid G) \nonumber \\
&=\sum_{\pi}P(O \mid \pi,G)P(\pi \mid G) \nonumber \\
&= P(O \mid G) \nonumber
\end{align}

\vspace{6em}
\section{Appendix B: Sampling Supporter Actions}
Extracting \emph{fact observation probabilities} relies on obtaining sets of supporter actions that constitute relaxed plans, and evaluating the probability that each fact is supported by at least one action in a relaxed plan.
We provide the algorithm for sampling supporter actions as presented in \cite{wilken2024fact} in Alg.~\ref{alg:sample_supporters1}. In red we highlight portions of the original algorithm we modified in our implementation in Alg.~\ref{alg:sample_supporters2}. First, we directly insert all goal condition facts into the set $C$ of facts that must be achieved (line 5), and thus search for achievers for the entire goal statement in a single pass, as opposed to separately seeking achievers for each goal fact and merging achievers for different goal facts together at a later stage. This is done to improve overall computational cost, and account for actions that achieve more than one relevant fact for the entire goal statement. Our second modification involves selecting the set of candidate supporter actions. Rather than all supporter actions in the earliest layer of the relaxed planning graph in which they occur, we select the set of valid supporter actions with minimum $h^{add}$ heuristic value (line 13). 

As mentioned in the main text, when estimating $P(O \mid \neg G)$, the algorithm seeks a relaxed plan that achieves all non-goal facts instead, using the extracted supporter actions to estimate the probability of occurrence of each fact.

\begin{algorithm} [H]
\caption{Supporter Action Sampling - Original.}\label{alg:sample_supporters1}
\begin{algorithmic}[1]
\Function{SampleRelevantActions}{$g_i$, RPG, $s_0$, $N$}
    \State $count \gets \{\}$ \Comment{Map from action to count in samples}
    \State $samples \gets [\,]$ \Comment{List of generated supporter sets}
    \For{$i \in \text{range}(0, N)$}
        \color{red}
        \State $C \gets g_i$ \Comment{Facts to be supported}
        \color{black}
        \State $found \gets \emptyset$
        \State $sups \gets \emptyset$
        \For{$t = \text{RPG.levels}$ \textbf{to} $0$}
            \State $newC \gets \emptyset$
            \While{$|C| > 0$}
                \State $p \gets C.\text{pop}()$
                \State $psups \gets \emptyset$
                \color{red}
                \For{$t_2 = 0$ \textbf{to} $t$}
                    \ForAll{$a \in \text{RPG.level}(t_2)$}
                        \If{$|\text{add}(a) \cap \{p\}| > 0$}
                            \State $psups \gets psups \cup \{a\}$
                        \EndIf
                    \EndFor
                    \If{$|psups| > 0$}
                        \State \textbf{break}
                    \EndIf
                \EndFor
                \color{black}
                \State $psups \gets$ \Call{minCount}{$psups$, $count$}
                \State $a \gets$ \Call{random}{$psups$}
                \State $found \gets found \cup \{p\}$
                \State $C \gets C \setminus \{p\}$
                \State $sups \gets sups \cup \{a\}$
                \State $count[a] \gets count[a] + 1$
                \ForAll{$n \in \text{pre}(a)$}
                    \If{$n \notin s_0 \land n \notin found \land n \notin C$}
                        \State $newC \gets newC \cup \{n\}$
                    \EndIf
                \EndFor
                \ForAll{$r \in \text{add}(a)$}
                    \State $C \gets C \setminus \{r\}$
                    \State $newC \gets newC \setminus \{r\}$
                \EndFor
            \EndWhile
            \State $C \gets C \cup \{newC\}$
        \EndFor
        \State $samples \gets$ $samples$.\text{add}($sups$)
    \EndFor
    \State \Return $samples$
\EndFunction
\end{algorithmic}
\end{algorithm}

\begin{algorithm}[H]
\caption{Supporter Action Sampling - Ours.}\label{alg:sample_supporters2}
\begin{algorithmic}[1]
\Function{SampleRelevantActions}{$G$, RPG, $s_0$, $N$}
    \State count $\gets \{\}$ \Comment{Map from action to count in samples}
    \State samples $\gets [\,]$ \Comment{List of generated supporter sets}
    \For{$i \in \text{range}(0, N)$}
        \color{red}
        \State \textbf{for} $g_i \in G$ \textbf{do} $C \gets g_i$ \Comment{Support all goal facts}
        % \State $C \gets g_i$ \Comment{Facts to be supported}
        \color{black}
        \State $found$ $\gets \emptyset$
        \State $sups$ $\gets \emptyset$
        \For{$t = \text{RPG.levels}$ \textbf{to} $0$}
            \State $newC \gets \emptyset$
            \While{$|C| > 0$}
                \State $p \gets C.\text{pop}()$
                \State $psups$ $\gets \emptyset$
                \color{red} 
                \State $psups$ $\gets$ \Call{$\scriptstyle \text{MinHAddSupporters}$}{$p$, $t$, RPG}
                \color{black}
                \State $psups$ $\gets$ \Call{minCount}{$psups$, $count$}
                \State $a \gets$ \Call{random}{$psups$}
                \State $found$ $\gets$ $found$ $\cup \{p\}$
                \State $C \gets C \setminus \{p\}$
                \State $sups$ $\gets$ $sups$ $\cup \{a\}$
                \State $count[a] \gets count[a] + 1$
                \ForAll{$n \in \text{pre}(a)$}
                    \If{$n \notin s_0 \land n \notin found \land n \notin C$}
                        \State $newC \gets newC \cup \{n\}$
                    \EndIf
                \EndFor
                \ForAll{$r \in \text{add}(a)$}
                    \State $C \gets C \setminus \{r\}$
                    \State $newC \gets newC \setminus \{r\}$
                \EndFor
            \EndWhile
            \State $C \gets C \cup \{newC\}$
        \EndFor
        \State $samples$ $\gets$ $samples$.\text{add}($sups$)
    \EndFor
    \State \Return $samples$
\EndFunction
\end{algorithmic}
\end{algorithm}

\section{Appendix C: Experimental Details}

\subsection{Number of sampled relaxed plans}
We note that altering the number of sampled relaxed plans in the fact observation probability extraction phase between 10, 100, 1000, and 10000 samples does not impact coverage meaningfully with this technique in empirical tests, so we did not focus on evaluating this aspect in performance benchmarks. We hypothesize that the quality of estimates obtained from relaxed plans, which approximate an inherently biased distribution rather than that of the non-relaxed problem, are not informative enough to gain any benefit from larger sample counts. It mainly affects the computational cost of the information extraction phase, as such we kept it at 100 to have a meaningful number of samples, while retaining low computation times. 

In several problems, the \emph{OL-UTP} version of the heuristic re-weights sampled plans by placing most of the weight ($>$90\%) on a single plan. Given the lower coverage but greater Agile score of this variant, it seems to induce a more ``high-risk high-reward" behaviour: it prioritises relaxed plans that are estimated to lead to solutions with fewer node generations, solving problems earlier when they prove valid, but potentially failing more when they do not.

\subsection{Measuring Agile score}
The Agile score is calculated for all planners using the overall process runtime provided by the Lab environment \cite{seipp2017downward}. This is done to obtain less biased estimates, without relying on the search runtime measured by the planning libraries themselves, which may vary in when they start or stop measurements.

\subsection{BFWS Variants}

\begin{itemize}
    \item BFWS($f5$): BFWS solver with evaluation function $f5:=\langle w_{\#r,\#g}, \#g \rangle$, where $\#r$ is the $\#r$ partition function from \cite{lipovetzky2017best}, and $\#g$ is the goal count heuristic. Remaining ties are broken by path length.
    \item BFWS($f5$)-$ol_r^{rp}$: BFWS($f5$) where third ties are broken by $ol_r^{rp}$, and remaining ties are broken by path length. BFWS($f5$)-$ol_r^{rp}$-UTP adopts $ol_r^{rp}$ with an additional UTP sample relaxed plan weighting function.
    \item BFWS($f5$)-$op_r^{rp}$: BFWS($f5$) where third ties are broken by $op_r^{rp}$, and remaining ties are broken by path length.
    \item BFWS($f5$)-Landmarks: BFWS solver with evaluation function $f5:=\langle w_{\#r,lm}, lm \rangle$, where $lm$ is the Landmarks heuristic \cite{richter2008landmarks}. Remaining ties are broken by path length.
    \item BFWS($f5$)$_t$-Landmarks-$ol_r^{rp}$: BFWS($f5$)-Landmarks where third ties are broken by $ol_r^{rp}$, and remaining ties are broken by path length. A key consideration is that it automatically reverts to using $\#g$ instead of the landmarks heuristic if it detects more than 100 facts in the goal condition. This is a practical consideration that is done to limit the number of problem partitions created by the planner, as traces used by $ol_r^{rp}$ start only at the most recent heuristic improvement. If there are too many improvements, then $ol_r^{rp}$ is restarted too often and does not inform the search. Since the number of landmarks is greater or equal to the number of goal facts, switching helps reduce this problem in such cases.
    % \item BFWS($f5$)$_t$: BFWS($f5$) solver that substitutes the normal open list with a trimmed open list \cite{rosa2024count}. This open list variant limits the maximum depth allowed of the binary heap, limiting memory usage. The maximum binary heap depth allowed by the open list is a hyperparameter, and is set to 18, following results in \cite{rosa2024count}. All BFWS$_t$($f5$) variants are the same as variants described above, but also adopt a trimmed open list.
    \item BFWS($f5$)$_t$-ol$_r^{rp}$-Dual: A \defterm{dual-strategy} solver, where a frontend solver attempts to solve the problem and, if it fails, it falls back to a backend solver. The frontend solver is BFWS($f5$)$_t$-Landmarks-$ol_r^{rp}$, the backend is the backend solver of \emph{Dual-BFWS} \cite{lipovetzky2017best}. This configuration is the same as \emph{BFNoS-Dual} \cite{rosa2024count} albeit substituting the BFNoS frontend with our proposed BFWS($f5$)$_t$-Landmarks-$ol_r^{rp}$. As such, like \emph{BFNoS-Dual}, it adopts both time and memory thresholds to signal the fallback of the frontend solver. We set the time thresholds to 1500 sec and the memory threshold to 6000 MB. We use as a reference the 1600 sec and 6000 MB thresholds of \emph{BFNoS-Dual}, but reduce the time threshold due to the faster solution times of our proposed solver compared to BFNoS.
\end{itemize}

\subsection{Sources and commands for Baseline Planners}

\paragraph{BFWS, Dual-BFWS, Approximate BFWS.} Run on LAPKT \cite{lapkt}.

\begin{verbatim}
BFWS --grounder FD -d <domain> 
    -p <problem> --search_type BFWS-f5 
\end{verbatim}

\begin{verbatim}
BFWS --grounder FD -d <domain> 
    -p <problem> --search_type DUAL-BFWS
\end{verbatim}

\begin{verbatim}
Approximate_BFWS --grounder Tarski
    -d <domain> -p <problem>  
    --seed <seed>
\end{verbatim}

\paragraph{BFNoS-Dual} Run on LAPKT-BFNoS \cite{lapkt-bfnos}.

\begin{verbatim}
BFWS --grounder FD
    -d <domain> -p <problem>  
    --search_type BFNOS
    --fallback_backend
    --backend_type DUAL-BFWS
    --time_limit 1600
    --memory_limit 6000
    --tol_seed <seed>
\end{verbatim}

\paragraph{LAMA} Run on the Fast Downward planning system \cite{helmert2006fast}.

\begin{verbatim}
--alias lama-first <domain> <problem>
\end{verbatim}

\paragraph{Scorpion-Maidu} We run a ``first" version of Scorpion Maidu \cite{correa-et-al-ipc2023c}, which halts after finding a solution rather than improving the plan, from the IPC-2023 branch of the code base \cite{maidu2023ipc}.
\begin{verbatim}
<domain> <problem>
--evaluator 'hlm=lmcount(
lm_factory=lm_reasonable_orders_hps(
    lm_rhw()),
transform=adapt_costs(one),pref=false)' 
--evaluator 
    'hff=ff(transform=adapt_costs(one))' 
--search 'lazy(alt([single(hff), 
single(hff, pref_only=true),
single(hlm), 
single(hlm, pref_only=true), 
type_based([hff, g()]), 
novelty_open_list(novelty(width=2, 
consider_only_novel_states=true, 
reset_after_progress=True), 
break_ties_randomly=False, 
handle_progress=move)],
boost=1000),preferred=[hff,hlm], 
cost_type=one,reopen_closed=false)'
\end{verbatim}

\section{Appendix D: Comparison Graphs and Coverage Tables}

\begin{figure*}[t]
  \centering

  \begin{subfigure}{0.32\textwidth}
    \centering
    \includegraphics[width=\linewidth]{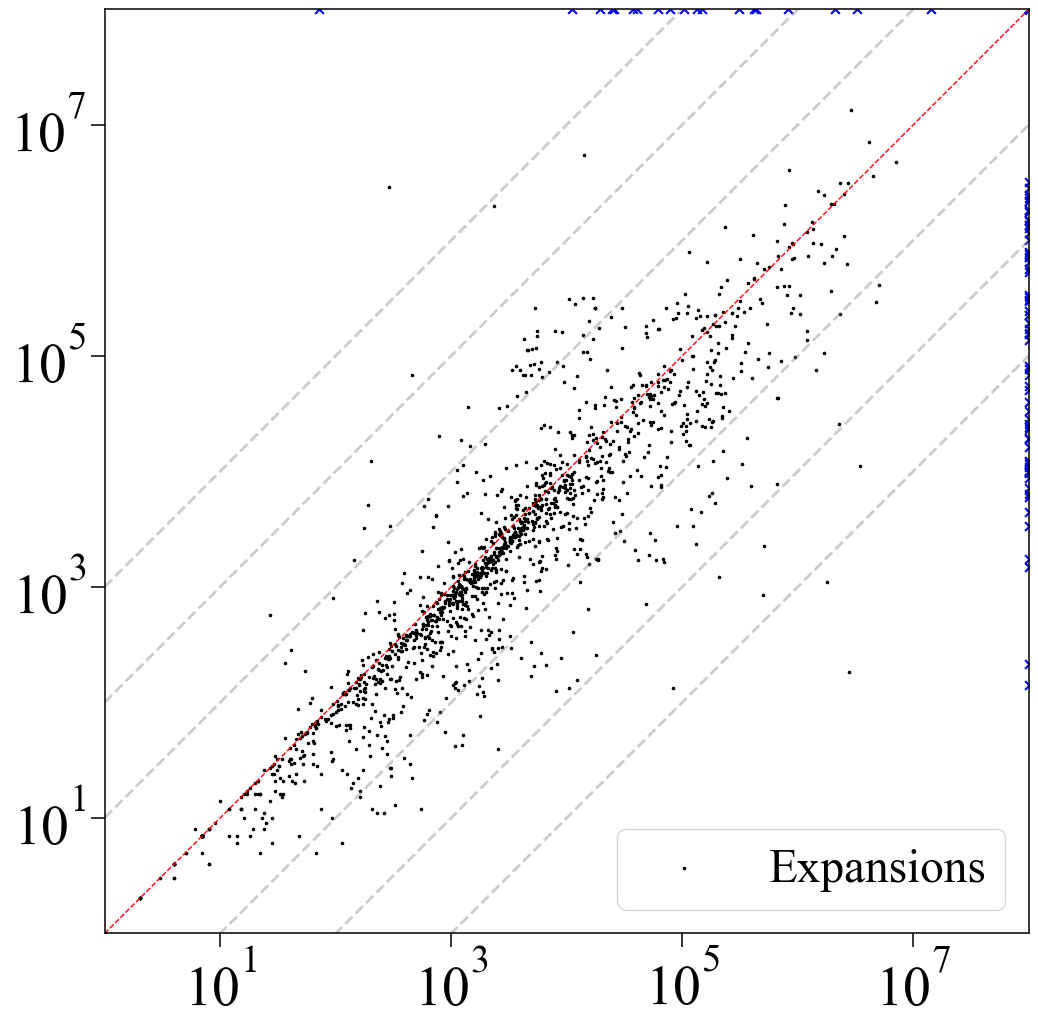}
    \caption{BFWS vs.\ BFWS-ol$_r^{rp}$}
    \label{fig:sub11}
  \end{subfigure}
  \hfill
  \begin{subfigure}{0.32\textwidth}
    \centering
    \includegraphics[width=\linewidth]{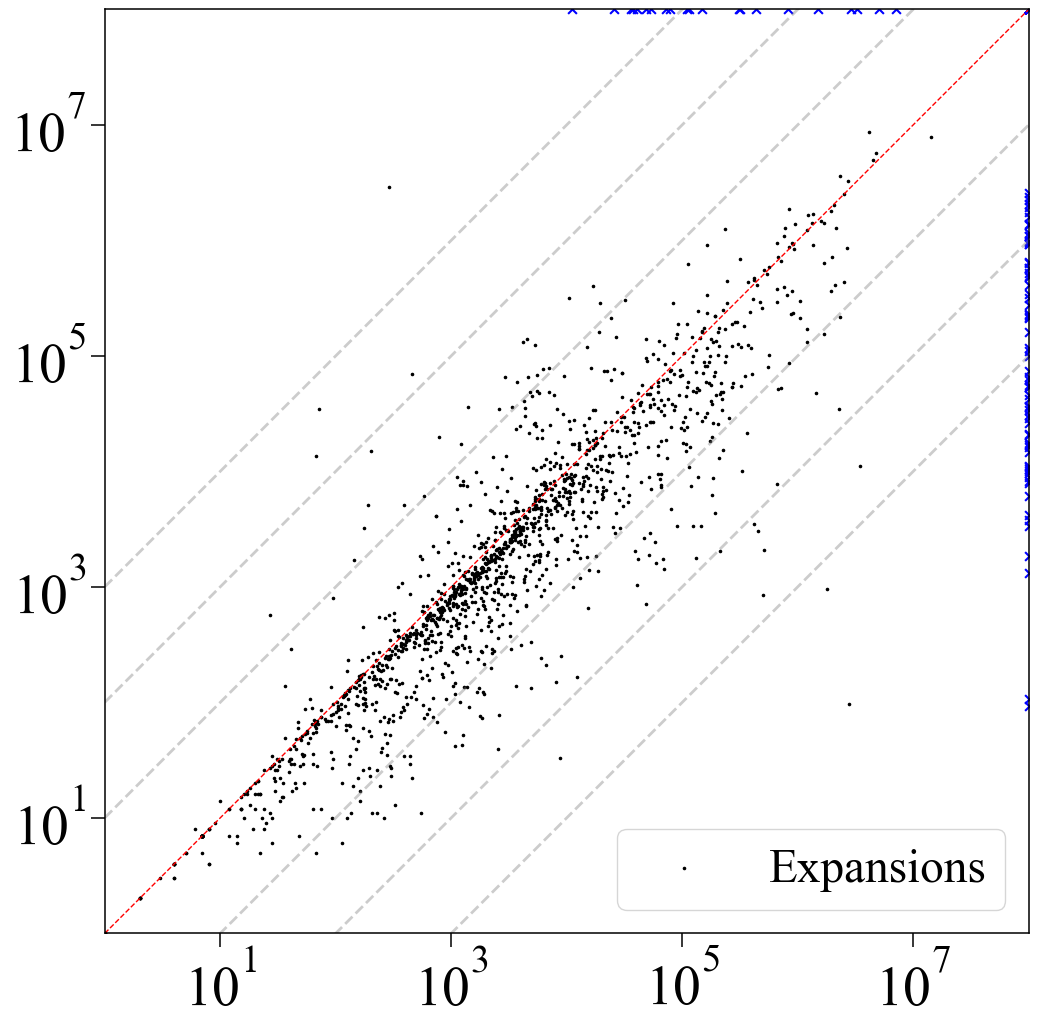}
    \caption{BFWS vs.\ BFWS-op$_r^{rp}$}
    \label{fig:sub12}
  \end{subfigure}
  \hfill
  \begin{subfigure}{0.32\textwidth}
    \centering
    \includegraphics[width=\linewidth]{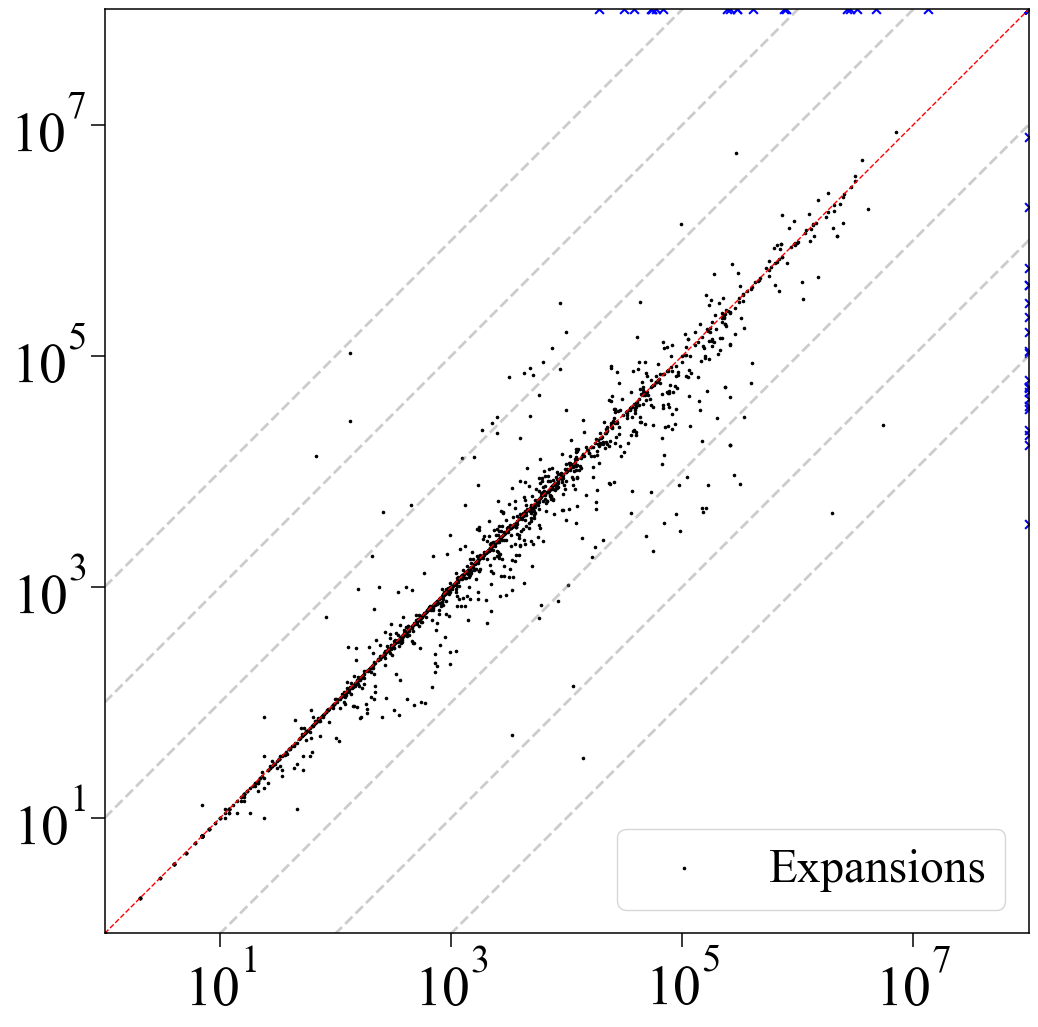}
    \caption{BFWS-ol$_r^{rp}$ vs.\ BFWS-op$_r^{rp}$}
    \label{fig:sub13}
  \end{subfigure}

  \caption{Number of expansions (x-axis vs.\ y-axis) (lower is better). Comparing BFWS($f5$), BFWS($f5$)-ol$_r^{rp}$, and BFWS($f5$)-op$_r^{rp}$. The high-density regions lie below the red line, indicating a systematic improvement over the baseline BFWS($f5$).}
  \label{fig:count_prob_analysis1}
\end{figure*}

\begin{figure*}[t]
  \centering

  \begin{subfigure}{0.32\textwidth}
    \centering
    \includegraphics[width=\linewidth]{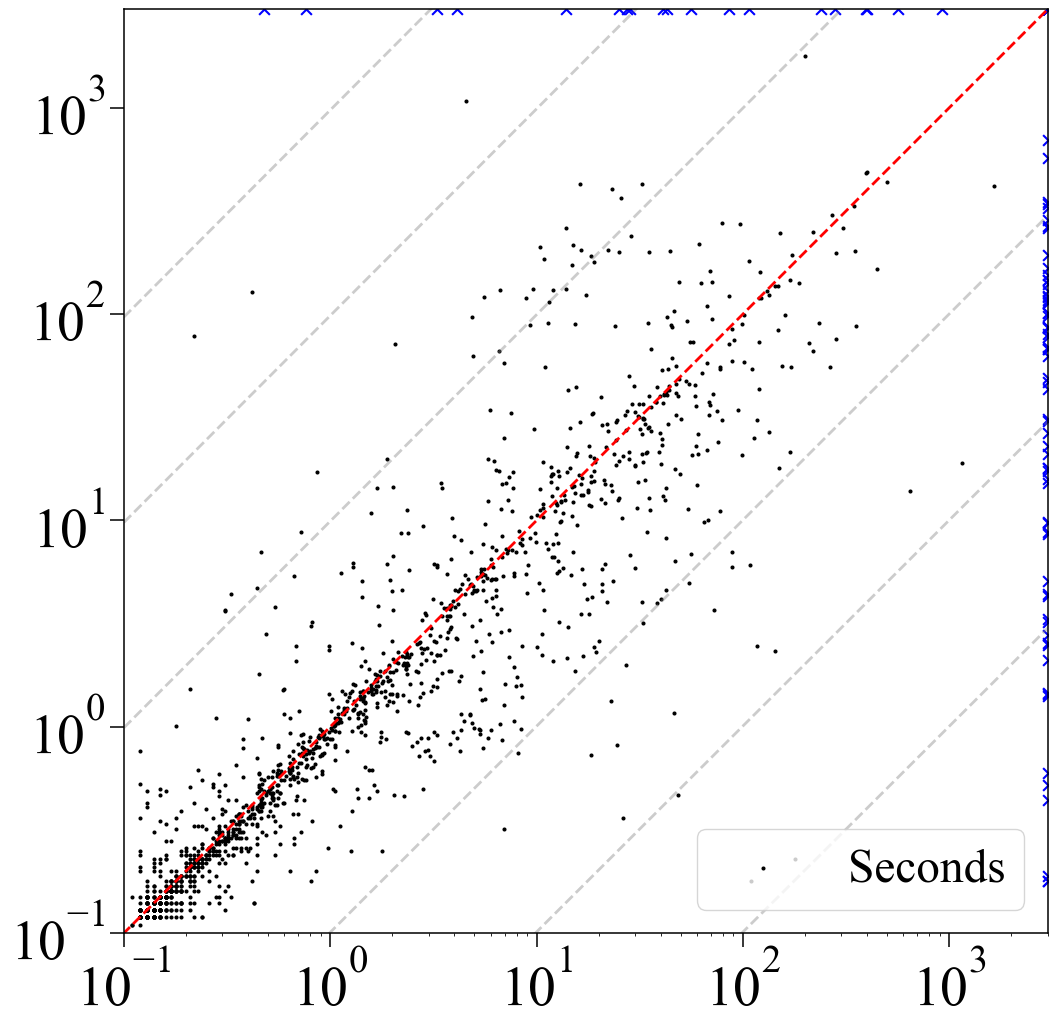}
    \caption{BFWS vs.\ BFWS-ol$_r^{rp}$}
    \label{fig:sub21}
  \end{subfigure}
  \hfill
  \begin{subfigure}{0.32\textwidth}
    \centering
    \includegraphics[width=\linewidth]{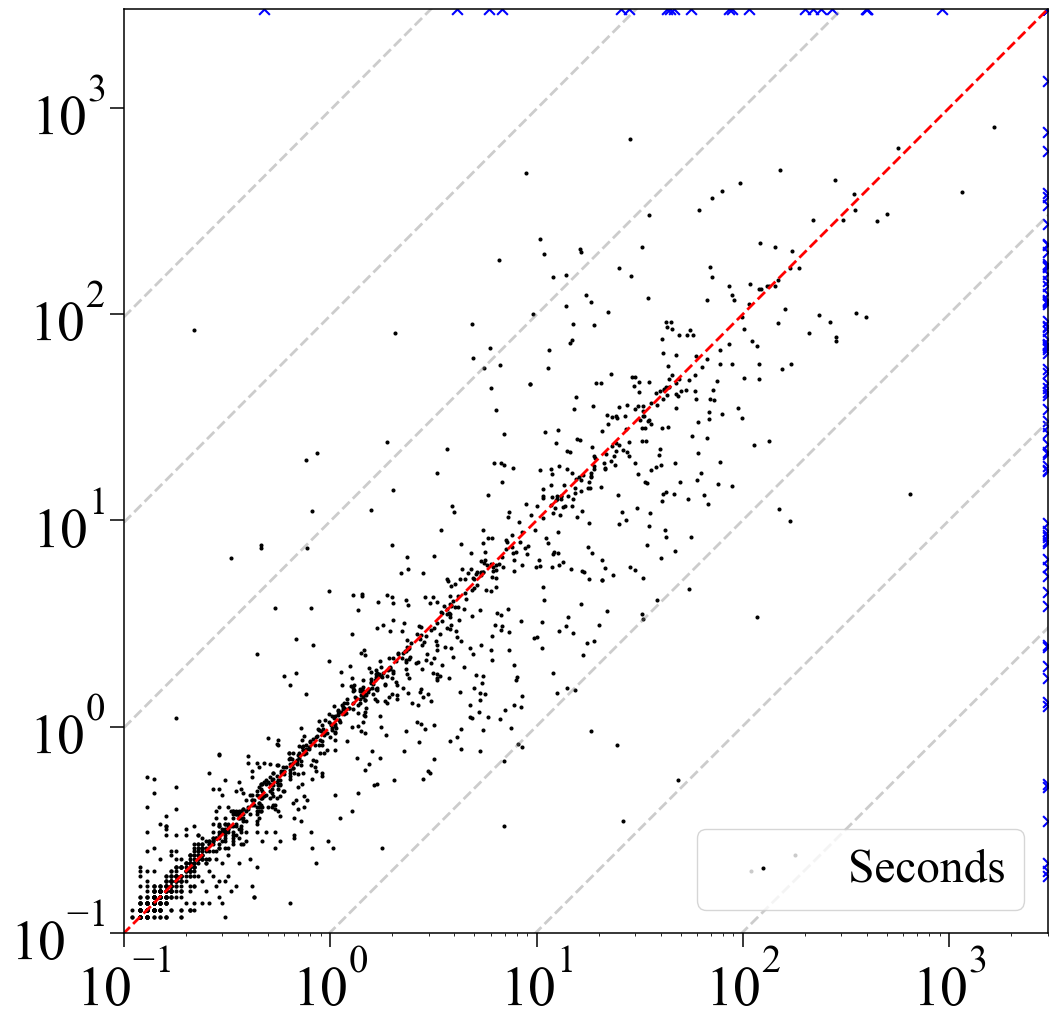}
    \caption{BFWS vs.\ BFWS-op$_r^{rp}$}
    \label{fig:sub22}
  \end{subfigure}
  \hfill
  \begin{subfigure}{0.32\textwidth}
    \centering
    \includegraphics[width=\linewidth]{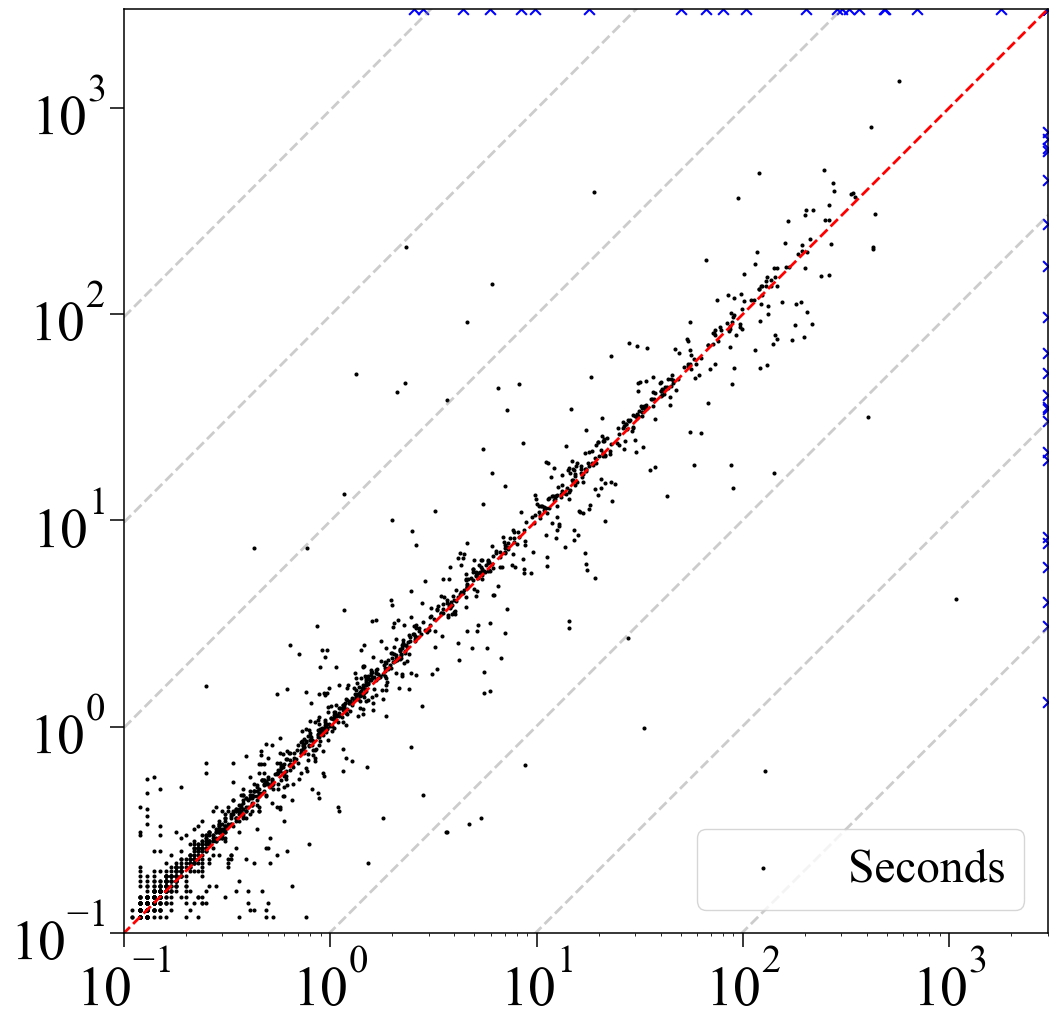}
    \caption{BFWS-ol$_r^{rp}$ vs.\ BFWS-op$_r^{rp}$}
    \label{fig:sub23}
  \end{subfigure}

  \caption{Total solver time (x-axis vs.\ y-axis) (lower is better). Comparing BFWS($f5$), BFWS($f5$)-ol$_r^{rp}$, and BFWS($f5$)-op$_r^{rp}$.}
  \label{fig:count_prob_analysis2}
\end{figure*}

\begin{figure*}[t]
  \centering

  \begin{subfigure}{0.32\textwidth}
    \centering
    \includegraphics[width=\linewidth]{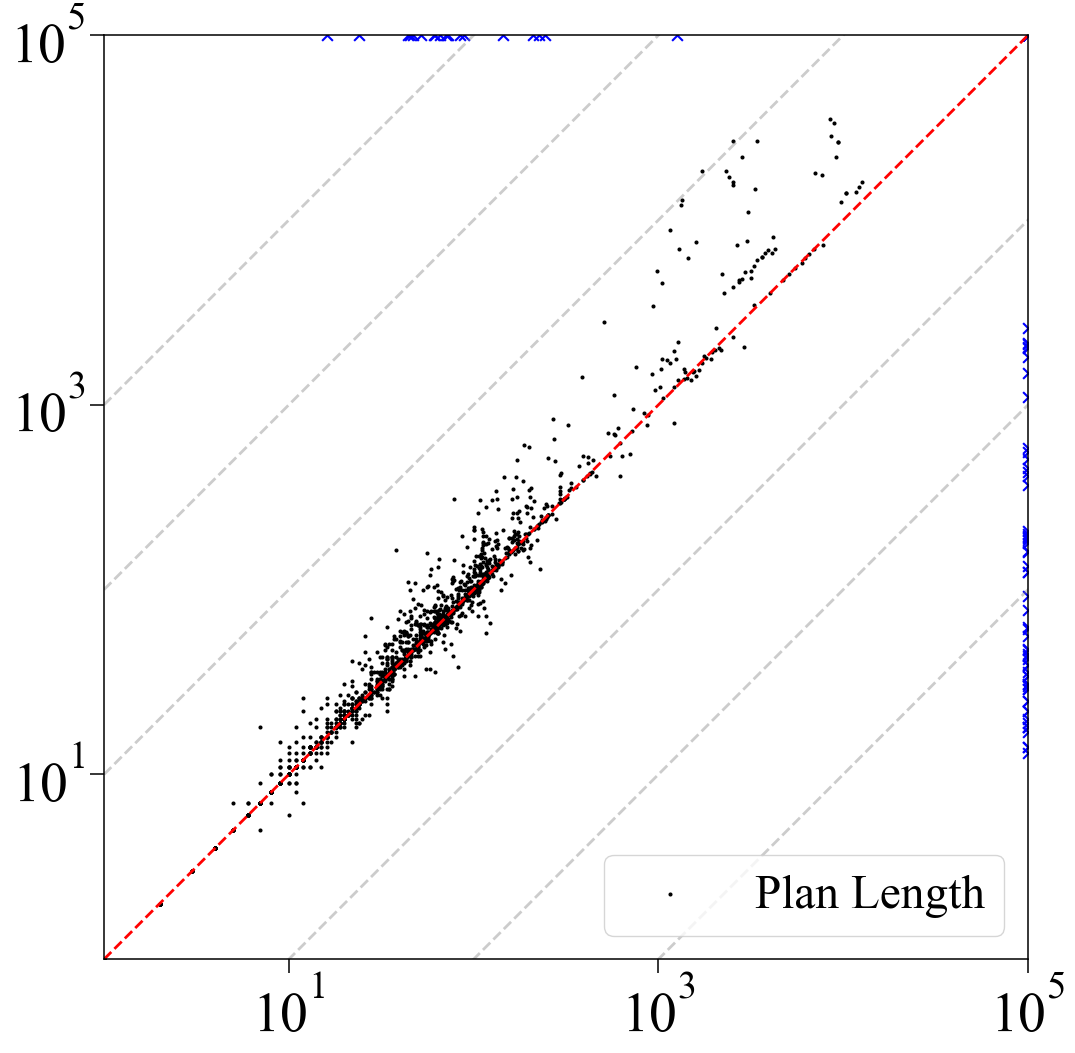}
    \caption{BFWS vs.\ BFWS-ol$_r^{rp}$}
    \label{fig:sub31}
  \end{subfigure}
  \hfill
  \begin{subfigure}{0.32\textwidth}
    \centering
    \includegraphics[width=\linewidth]{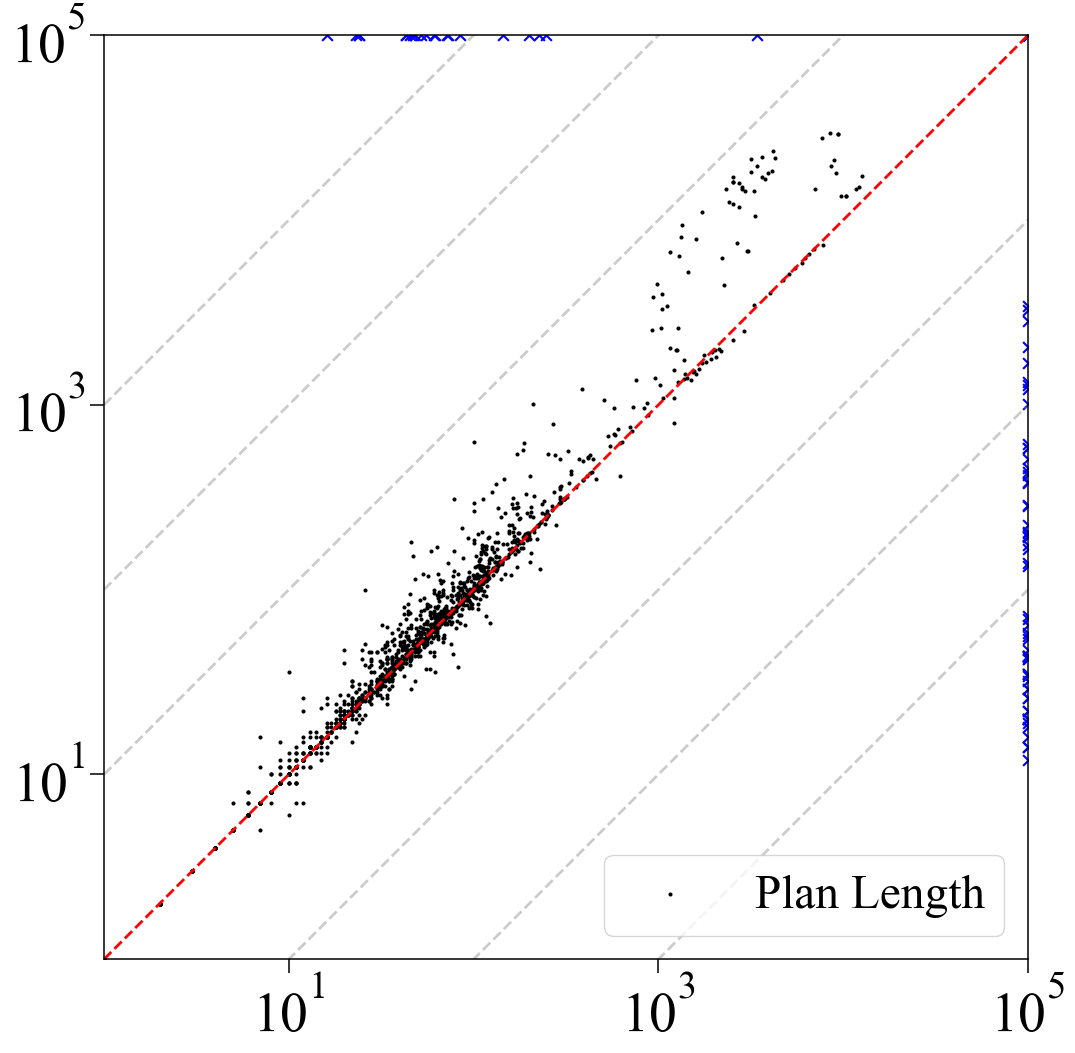}
    \caption{BFWS vs.\ BFWS-op$_r^{rp}$}
    \label{fig:sub32}
  \end{subfigure}
  \hfill
  \begin{subfigure}{0.32\textwidth}
    \centering
    \includegraphics[width=\linewidth]{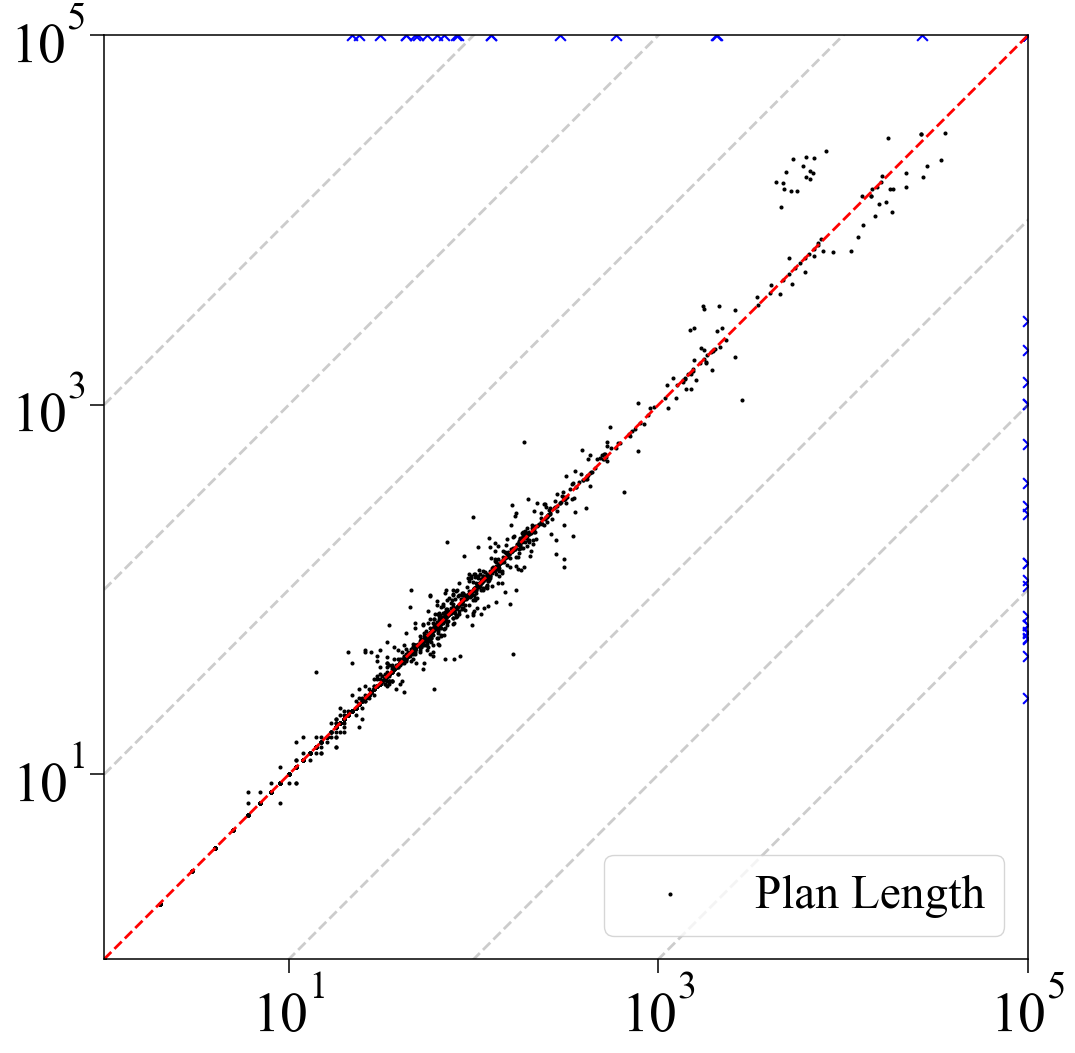}
    \caption{BFWS-ol$_r^{rp}$ vs.\ BFWS-op$_r^{rp}$}
    \label{fig:sub33}
  \end{subfigure}

  \caption{Plan cost (Unit action costs) (x-axis vs.\ y-axis) (lower is better). Comparing BFWS($f5$), BFWS($f5$)-ol$_r^{rp}$, and BFWS($f5$)-op$_r^{rp}$.}
  \label{fig:count_prob_analysis3}
\end{figure*}

\clearpage
\newpage

\begin{table*}[tbhp]
    \centering
    % \small
    \resizebox{!}{0.5\textheight}{
    \begin{tabular}{|l|c|c|c|c|c|c|} \hline
         \rule{0pt}{2.2ex}Domain & BFWS($f5$) &  BFWS($f5$)-ol$_r^{rp}$ &  BFWS($f5$)-op$_r^{rp}$ \\ \hline
        agricola-sat18-strips (20) &\makecell{10}& \makecell{{12$\pm$0.6}} &  \makecell{ 12$\pm$0.8 }\\ 
        airport (50)              &\makecell{{47}}& \makecell{{47$\pm$0.6}} &\makecell{ 48$\pm$0.0 }\\
        assembly (30)             &\makecell{30}& \makecell{30$\pm$0.5}&\makecell{ 29$\pm$0.5 } \\
        barman-sat14-strips (20)  &\makecell{20}& \makecell{20$\pm$0.0}&\makecell{ 20$\pm$0.0 }\\ 
        blocks (35)               &\makecell{35}& \makecell{35$\pm$0.0} &\makecell{ 35$\pm$0.0 }\\
        caldera-sat18-adl (20)    &\makecell{15}& \makecell{{18$\pm$0.6}}&\makecell{ 16$\pm$1.0 }\\ 
        cavediving-14-adl (20)    &\makecell{{7}}& \makecell{{7$\pm$0.0}}&\makecell{ 7$\pm$0.0 } \\
        childsnack-sat14-strips (20)&\makecell{0}& \makecell{0$\pm$0.0} &\makecell{ 0$\pm$0.0 }\\ 
        citycar-sat14-adl (20)    &\makecell{{5}}& \makecell{{5$\pm$0.6}}&\makecell{ 4$\pm$0.4 } \\ 
        data-network-sat18-strips (20) &\makecell{11}& \makecell{{17$\pm$0.8}}&\makecell{ 15$\pm$1.1 }\\ 
        depot (22)                &\makecell{{22}}& \makecell{{22$\pm$0.0}} &\makecell{ 22$\pm$0.0 }\\ 
        driverlog (20)            &\makecell{20}& \makecell{20$\pm$0.0}&\makecell{ 20$\pm$0.0 }\\
        elevators-sat11-strips (20) &\makecell{20}& \makecell{20$\pm$0.0} &\makecell{ 20$\pm$0.4 }\\ 
        flashfill-sat18-adl (20)  &\makecell{12}& \makecell{16$\pm$0.5}&\makecell{ 17$\pm$0.5 }\\
        floortile-sat14-strips (20) &\makecell{2}& \makecell{2$\pm$0.0} &\makecell{ 2$\pm$0.0 }\\
        folding (20)              &\makecell{8}& \makecell{8$\pm$0.8}&\makecell{ 8$\pm$0.5 }\\
        freecell (80)             &\makecell{{80}}& \makecell{{80$\pm$0.0}}&\makecell{ 80$\pm$0.0 }\\ 
        ged-sat14-strips (20)     &\makecell{20}& \makecell{20$\pm$0.0}&\makecell{ 20$\pm$0.0 }\\
        grid (5)                 &\makecell{5}& \makecell{5$\pm$0.0}&\makecell{ 5$\pm$0.0 }\\
        gripper (20)              &\makecell{20}& \makecell{20$\pm$0.0} &\makecell{ 20$\pm$0.0 }\\ 
        hiking-sat14-strips (20)  &\makecell{{11}}& \makecell{{9$\pm$1.5}} &\makecell{ 10$\pm$2.0 }\\ 
        labyrinth (20)            &\makecell{15}& \makecell{{15$\pm$0.0}} &\makecell{ 15$\pm$0.0 }\\ 
        logistics00 (28)          &\makecell{28}& \makecell{28$\pm$0.0} &\makecell{ 28$\pm$0.0 }\\ 
        maintenance-sat14-adl (20) &\makecell{{17}}& \makecell{{16$\pm$0.5}}& \makecell{ 17$\pm$0.4 }\\ 
        miconic (150)              &\makecell{150}& \makecell{150$\pm$0.0}&\makecell{ 150$\pm$0.0 }\\ 
        movie (30)                &\makecell{30}& \makecell{30$\pm$0.0} &\makecell{ 30$\pm$0.0 }\\ 
        mprime (35)               &\makecell{30}& \makecell{34$\pm$0.0}&\makecell{ 34$\pm$0.0 }\\ 
        mystery (30)              &\makecell{19}& \makecell{{18$\pm$0.5}} & \makecell{ 18$\pm$0.0 }\\ 
        nomystery-sat11-strips (20) &\makecell{16}& \makecell{15$\pm$0.8}&\makecell{ 14$\pm$0.4 }\\ 
        nurikabe-sat18-adl (20)   &\makecell{16}& \makecell{16$\pm$0.8}&\makecell{ 16$\pm$0.8 }\\ 
        openstacks-sat14-strips (20) &\makecell{20}& \makecell{20$\pm$0.0}&\makecell{ 20$\pm$0.0 }\\ 
        organic-synthesis-split-sat18-strips (20) & \makecell{5}&\makecell{5$\pm$0.0}&  \makecell{ 6$\pm$0.0 }\\ 
        parcprinter-sat11-strips (20) &\makecell{9} &\makecell{16$\pm$0.0}& \makecell{ 16$\pm$0.5 }\\ 
        parking-sat14-strips (20) &\makecell{20}& \makecell{20$\pm$0.0}& \makecell{ 20$\pm$0.0 }\\ 
        pathways (30)             &\makecell{24}& \makecell{29$\pm$0.5}&\makecell{ 29$\pm$0.8 }\\ 
        pegsol-sat11-strips (20)  &\makecell{19}& \makecell{19$\pm$0.0}& \makecell{ 20$\pm$0.0 }\\ 
        pipesworld-notankage (50)  &\makecell{{50}}& \makecell{{50$\pm$0.0}}& \makecell{ 50$\pm$0.0 }\\ 
        pipesworld-tankage (50)   &\makecell{44}& \makecell{{44$\pm$0.8}}& \makecell{ 44$\pm$1.3 }\\ 
        psr-small (50)             &\makecell{47}& \makecell{48$\pm$0.0}&\makecell{ 47$\pm$0.0 }\\ 
        quantum-layout (20)       &\makecell{20}& \makecell{20$\pm$0.0}& \makecell{ 20$\pm$0.0 } \\ 
        recharging-robots (20)     &\makecell{{14}}& \makecell{{14$\pm$0.5}}&  \makecell{ 14$\pm$0.4 } \\ 
        ricochet-robots (20)      &\makecell{{1}}& \makecell{20$\pm$0.0}&\makecell{ 19$\pm$0.0 }\\ 
        rovers (40)               &\makecell{39}&  \makecell{40$\pm$0.6}& \makecell{ 40$\pm$0.4 }\\ 
        rubiks-cube (20)          &\makecell{5}&  \makecell{5$\pm$0.0}& \makecell{ 5$\pm$0.0 }\\ 
        satellite (36)            &\makecell{28}&  \makecell{30$\pm$0.0}& \makecell{ 30$\pm$0.0 }\\ 
        scanalyzer-sat11-strips (20) &\makecell{20}&  \makecell{20$\pm$0.5}&\makecell{ 20$\pm$0.0 }\\ 
        schedule (150)             &\makecell{149}&   \makecell{150$\pm$0.0}&\makecell{ 150$\pm$0.0 }\\ 
        settlers-sat18-adl (20)   &\makecell{10}&  \makecell{10$\pm$1.1}& \makecell{ 9$\pm$0.5 }\\ 
        slitherlink (20)          &\makecell{{4}}&  \makecell{{4$\pm$0.5}} & \makecell{ 4$\pm$0.5 }\\ 
        snake-sat18-strips (20)   &\makecell{{18}}&  \makecell{{19$\pm$1.0}}& \makecell{ 18$\pm$0.9 }\\ 
        sokoban-sat11-strips (20) &\makecell{15}&  \makecell{14$\pm$0.6} & \makecell{ 13$\pm$1.3 }\\ 
        spider-sat18-strips (20)  &\makecell{14}&\  \makecell{12$\pm$1.4}&  \makecell{ 13$\pm$0.8 }\\ 
        storage (30)              &\makecell{{29}}&  \makecell{{30$\pm$0.5}}& \makecell{ 29$\pm$0.5 } \\ 
        termes-sat18-strips (20)  &\makecell{9}&  \makecell{8$\pm$0.5}& \makecell{ 9$\pm$0.4 }\\ 
        tetris-sat14-strips (20)  &\makecell{{20}}&  \makecell{{20$\pm$0.0}}& \makecell{ 20$\pm$0.0 }\\ 
        thoughtful-sat14-strips (20) &\makecell{{20}}&  \makecell{{20$\pm$0.0}}&  \makecell{ 20$\pm$0.0 }\\ 
        tidybot-sat11-strips (20)  &\makecell{{20}}&  \makecell{{20$\pm$0.5}}&  \makecell{ 20$\pm$0.4 }\\ 
        tpp (30)                  &\makecell{29} &  \makecell{30$\pm$0.6}& \makecell{ 30$\pm$0.0 }\\
        transport-sat14-strips (20) &\makecell{{20}}&  \makecell{{20$\pm$0.0}}&  \makecell{ 20$\pm$0.0 }\\
        trucks-strips (30)         &\makecell{8}&  \makecell{9$\pm$0.8}& \makecell{ 9$\pm$0.5 }\\ 
        visitall-sat14-strips (20) &\makecell{20}&  \makecell{20$\pm$0.0}& \makecell{ 20$\pm$0.0 }\\ 
        woodworking-sat11-strips (20) &\makecell{{20}}&  \makecell{20$\pm$0.0}& \makecell{ 20$\pm$0.0 }\\ 
        zenotravel (20)           &\makecell{20}&  \makecell{20$\pm$0.0}& \makecell{ 20$\pm$0.0 }\\ \hline
         \rule{0pt}{2.2ex}\textbf{Coverage (1831)} &\makecell{1510}&  \makecell{1560$\pm$5.0}  & \makecell{1556$\pm$1.5} \\ \hline
         \rule{0pt}{2.2ex}\textbf{\% Score (100\%)} &\makecell{76.77}&  \makecell{80.20\%$\pm$0.35} & \makecell{79.90\%$\pm$0.15} \\ \hline
    \end{tabular}
    }
    \caption{Comparative performance analysis across the full set of benchmark domains. \textit{\% score} is the average of the \% of instances solved in each domain. Values for solvers with randomized components represent the mean and include the standard deviation across 5 measurements, using seeds from 0 to 4.}
    \label{tab:comparative_performance_large1}
\end{table*}

\begin{table*}[tbhp]
    \centering
    \resizebox{\textwidth}{!}{
    \begin{tabular}{|l||c|c|c|c|c||c|c|c|c|} \hline
        \rule{0pt}{2.2ex} Domain  & Dual- & Apx-BFWS &  LAMA- & Scorpion & BFNoS-Dual & BFWS$_t$-hlm &BFWS$_t$-hlm &  BFWS$_t$-hlm & BFWS$_t$-hlm\\
                & BFWS & (Tarski) & First & Maidu      &            & op$_r^{rp}$ & ol$_r^{rp}$ & ol$_r^{rp}$-UTP & ol$_r^{rp}$-Dual\\ 
        \hline
        agricola-sat18-strips   &\makecell{13}&\makecell{{18$\pm$0.6}}&\makecell{12}&\makecell{12}&\makecell{ 15$\pm$0.0 }        & \makecell{ 13$\pm$1.5 }&\makecell{14$\pm$1.0}&\makecell{14$\pm$0.6}&  \makecell{ 14$\pm$0.6 }    \\ 
        airport                 &\makecell{46}&\makecell{{47$\pm$0.6}}&\makecell{34}&\makecell{38}&\makecell{ 46$\pm$0.6 }       & \makecell{ 47$\pm$0.0 }&\makecell{47$\pm$0.0}&\makecell{47$\pm$0.0}&  \makecell{ 47$\pm$0.0 }    \\
        assembly                &\makecell{30}&\makecell{30$\pm$0.0}&\makecell{30}&\makecell{30}&\makecell{ 30$\pm$0.0 }         & \makecell{ 30$\pm$0.0 }&\makecell{30$\pm$0.0}&\makecell{30$\pm$0.6}&  \makecell{ 30$\pm$0.0 }    \\
        barman-sat14-strips     &\makecell{20}&\makecell{20$\pm$0.0}&\makecell{20}&\makecell{20}&\makecell{ 20$\pm$0.0 }          & \makecell{ 20$\pm$0.0 }&\makecell{20$\pm$0.0}&\makecell{20$\pm$0.0}&  \makecell{ 20$\pm$0.0 }     \\ 
        blocks                  &\makecell{35}&\makecell{35$\pm$0.0}&\makecell{35}&\makecell{35}&\makecell{ 35$\pm$0.0 }           &\makecell{ 35$\pm$0.0 }&\makecell{35$\pm$0.0}&\makecell{35$\pm$0.0}&  \makecell{ 35$\pm$0.0 }    \\
        caldera-sat18-adl       &\makecell{{18}}&\makecell{{19$\pm$0.6}}&\makecell{16}&\makecell{16}&\makecell{ 16$\pm$0.0 }       &\makecell{ 18$\pm$0.6 }&\makecell{18$\pm$0.6}&\makecell{18$\pm$0.6}&    \makecell{ 16$\pm$0.0 }  \\ 
        cavediving-14-adl       &\makecell{{8}}&\makecell{{8$\pm$0.6}}&\makecell{7}&\makecell{7}&\makecell{ 8$\pm$0.0 }            &\makecell{ 7$\pm$0.6 }&\makecell{7$\pm$0.6}&\makecell{{7$\pm$0.0}}&   \makecell{ 8$\pm$0.0 }    \\
        childsnack-sat14-strips &\makecell{{9}}&\makecell{5$\pm$1.5}&\makecell{6}&\makecell{6}&\makecell{ 8$\pm$0.0 }              &\makecell{ 0$\pm$0.6 }&\makecell{0$\pm$0.0}&\makecell{0$\pm$0.0}&   \makecell{ 8$\pm$0.6 }    \\ 
        citycar-sat14-adl       &\makecell{{20}}&\makecell{{20$\pm$0.0}}&\makecell{5}&\makecell{7}&\makecell{ 20$\pm$0.0 }         & \makecell{ 19$\pm$0.0 }&\makecell{19$\pm$0.6}&\makecell{{20$\pm$0.0}}&  \makecell{ 19$\pm$0.6 }   \\ 
        data-network-sat18-strips &\makecell{16}&\makecell{{19$\pm$0.0}}&\makecell{13}&\makecell{16}&\makecell{ 15$\pm$0.6 }       & \makecell{ 18$\pm$0.0 }&\makecell{18$\pm$0.0}&\makecell{18$\pm$0.0}&   \makecell{ 18$\pm$0.0 }   \\ 
        depot                   &\makecell{{22}}&\makecell{{22$\pm$0.0}}&\makecell{20}&\makecell{{22}}&\makecell{ 22$\pm$0.0 }     & \makecell{ 22$\pm$0.0 }&\makecell{{22$\pm$0.0}}&\makecell{{22$\pm$0.0}}&    \makecell{ 22$\pm$0.0 }    \\ 
        driverlog               &\makecell{20}&\makecell{20$\pm$0.0}&\makecell{20}&\makecell{20}&\makecell{ 20$\pm$0.0 }           &\makecell{ 20$\pm$0.0 }&\makecell{20$\pm$0.0}&\makecell{20$\pm$0.0}&   \makecell{ 20$\pm$0.0 }    \\
        elevators-sat11-strips  &\makecell{20}&\makecell{20$\pm$0.0}&\makecell{20}&\makecell{20}&\makecell{ 20$\pm$0.0 }           &\makecell{ 20$\pm$0.0 }&\makecell{20$\pm$0.0}&\makecell{20$\pm$0.0}&    \makecell{ 20$\pm$0.0 }   \\ 
        flashfill-sat18-adl     &\makecell{{17}}&\makecell{15$\pm$1.0}&\makecell{14}&\makecell{15}&\makecell{ 17$\pm$0.0 }         & \makecell{ 16$\pm$0.6 }&\makecell{17$\pm$0.6}&\makecell{{15$\pm$0.6}}&   \makecell{ 18$\pm$0.6 }     \\
        floortile-sat14-strips  &\makecell{2}&\makecell{2$\pm$0.0}&\makecell{2}&\makecell{2}&\makecell{ 2$\pm$0.0 }                &\makecell{ 2$\pm$0.0 }&\makecell{{1$\pm$0.6}}&\makecell{1$\pm$0.6}&   \makecell{ 2$\pm$0.0 }     \\ 
        folding                 &\makecell{5}&\makecell{5$\pm$0.6}&\makecell{{11}}&\makecell{{11}}&\makecell{ 9$\pm$0.0 }          &\makecell{ 9$\pm$0.6 }&\makecell{{8$\pm$0.6}}&\makecell{9$\pm$1.0}&    \makecell{ 8$\pm$0.6 }    \\
        freecell                &\makecell{{80}}&\makecell{{80$\pm$0.0}}&\makecell{79}&\makecell{{80}}&\makecell{ 80$\pm$0.0 }     & \makecell{ 80$\pm$0.0 }&\makecell{{80$\pm$0.0}}&\makecell{{80$\pm$0.0}}&   \makecell{ 80$\pm$0.0 }     \\ 
        ged-sat14-strips        &\makecell{20}&\makecell{20$\pm$0.0}&\makecell{20}&\makecell{20}&\makecell{ 20$\pm$0.0 }           &\makecell{ 20$\pm$0.0 }&\makecell{20$\pm$0.0}&\makecell{20$\pm$0.0}&   \makecell{ 20$\pm$0.0 }    \\
        grid                    &\makecell{5}&\makecell{5$\pm$0.0}&\makecell{5}&\makecell{5}&\makecell{ 5$\pm$0.0 }                &\makecell{ 5$\pm$0.0 }&\makecell{5$\pm$0.0}&\makecell{5$\pm$0.0}&   \makecell{ 5$\pm$0.0 }     \\ 
        gripper                 &\makecell{20}&\makecell{20$\pm$0.0}&\makecell{20}&\makecell{20}&\makecell{ 20$\pm$0.0 }           &\makecell{ 20$\pm$0.0 }&\makecell{20$\pm$0.0}&\makecell{20$\pm$0.0}&   \makecell{ 20$\pm$0.0 }     \\ 
        hiking-sat14-strips     &\makecell{18}&\makecell{{20$\pm$0.0}}&\makecell{{20}}&\makecell{{20}}&\makecell{ 20$\pm$0.0 }     &\makecell{ 20$\pm$0.0 }&\makecell{{20$\pm$0.0}}&\makecell{{20$\pm$0.0}}&   \makecell{ 20$\pm$0.0 }    \\ 
        labyrinth               &\makecell{5}&\makecell{{18$\pm$0.6}}&\makecell{1}&\makecell{0}&\makecell{ 15$\pm$0.6 }            &\makecell{ 15$\pm$0.0 }&\makecell{15$\pm$0.0}&\makecell{15$\pm$0.0}&   \makecell{ 15$\pm$0.0 }    \\ 
        logistics00             &\makecell{28}&\makecell{28$\pm$0.0}&\makecell{28}&\makecell{28}&\makecell{ 28$\pm$0.0 }           &\makecell{ 28$\pm$0.0 }&\makecell{28$\pm$0.0}&\makecell{28$\pm$0.0}&  \makecell{ 28$\pm$0.0 }    \\ 
        maintenance-sat14-adl   &\makecell{{17}}&\makecell{{17$\pm$0.0}}&\makecell{11}&\makecell{13}&\makecell{ 17$\pm$0.0 }       & \makecell{ 17$\pm$0.0 }&\makecell{17$\pm$0.0}&\makecell{{17$\pm$0.0}}&  \makecell{ 17$\pm$0.0 }    \\ 
        miconic                 &\makecell{150}&\makecell{150$\pm$0.0}&\makecell{150}&\makecell{150}&\makecell{ 150$\pm$0.0 }      &\makecell{ 150$\pm$0.0 }&\makecell{150$\pm$0.0}&\makecell{150$\pm$0.0}&    \makecell{ 150$\pm$0.0 }   \\ 
        movie                   &\makecell{30}&\makecell{30$\pm$0.0}&\makecell{30}&\makecell{30}&\makecell{ 30$\pm$0.0 }           &\makecell{ 30$\pm$0.0 }&\makecell{30$\pm$0.0}&\makecell{30$\pm$0.0}&  \makecell{ 30$\pm$0.0 }     \\ 
        mprime                  &\makecell{35}&\makecell{35$\pm$0.0}&\makecell{35}&\makecell{35}&\makecell{ 35$\pm$0.0 }           &\makecell{ 35$\pm$0.0 }&\makecell{35$\pm$0.0}&\makecell{35$\pm$0.0}& \makecell{ 35$\pm$0.0 }     \\ 
        mystery                 &\makecell{{19}}&\makecell{{19$\pm$0.0}}&\makecell{{19}}&\makecell{{19}}&\makecell{ 19$\pm$0.0 }   &\makecell{ 18$\pm$0.6 }&\makecell{{19$\pm$0.0}}&\makecell{{19$\pm$0.0}}&   \makecell{ 19$\pm$0.0 }    \\ 
        nomystery-sat11-strips  &\makecell{{19}}&\makecell{14$\pm$1.0}&\makecell{11}&\makecell{{19}}&\makecell{ 19$\pm$0.0 }       &\makecell{ 17$\pm$0.6 }&\makecell{15$\pm$0.6}&\makecell{{17$\pm$0.6}}&    \makecell{ 19$\pm$0.0 }   \\ 
        nurikabe-sat18-adl      &\makecell{14}&\makecell{18$\pm$0.6}&\makecell{9}&\makecell{11}&\makecell{ 16$\pm$0.6 }            &\makecell{ 14$\pm$0.0 }&\makecell{16$\pm$0.6}&\makecell{15$\pm$0.0}&  \makecell{ 16$\pm$0.6 }  \\ 
        openstacks-sat14-strips &\makecell{20}&\makecell{20$\pm$0.0}&\makecell{20}&\makecell{20}&\makecell{ 20$\pm$0.0 }           &\makecell{ 20$\pm$0.0 }&\makecell{20$\pm$0.0}&\makecell{20$\pm$0.0}&     \makecell{ 20$\pm$0.0 }  \\ 
        organic-synthesis-split-sat18-strips &\makecell{12}&\makecell{8$\pm$0.6}&\makecell{{14}}&\makecell{{14}}&\makecell{ 12$\pm$0.0 }   & \makecell{ 5$\pm$0.6 } &\makecell{{5$\pm$0.0}}&\makecell{4$\pm$0.5}&   \makecell{ 12$\pm$0.0 } \\ 
        parcprinter-sat11-strips &\makecell{16}&\makecell{10$\pm$0.6}&\makecell{{20}}&\makecell{{20}}&\makecell{ 20$\pm$0.0 }      &\makecell{ 15$\pm$0.0 }&\makecell{{16$\pm$0.0}}&\makecell{{16$\pm$0.6}}&    \makecell{ 20$\pm$0.0 }   \\ 
        parking-sat14-strips    &\makecell{20}&\makecell{20$\pm$0.0}&\makecell{20}&\makecell{20}&\makecell{ 20$\pm$0.0 }           &\makecell{ 20$\pm$0.0 }&\makecell{20$\pm$0.0}&\makecell{20$\pm$0.0}&    \makecell{ 20$\pm$0.0 }   \\ 
        pathways                &\makecell{{30}}&\makecell{29$\pm$1.5}&\makecell{23}&\makecell{25}&\makecell{ 30$\pm$0.0 }         &\makecell{ 30$\pm$0.6 }&\makecell{30$\pm$0.6}&\makecell{{29$\pm$0.6}}&    \makecell{ 30$\pm$0.0 }   \\ 
        pegsol-sat11-strips     &\makecell{20}&\makecell{20$\pm$0.0}&\makecell{20}&\makecell{20}&\makecell{ 20$\pm$0.0 }           &\makecell{ 20$\pm$0.0 }&\makecell{20$\pm$0.0}&\makecell{20$\pm$0.0}&    \makecell{ 20$\pm$0.0 }   \\ 
        pipesworld-notankage    &\makecell{{50}}&\makecell{{50$\pm$0.0}}&\makecell{43}&\makecell{45}&\makecell{ 50$\pm$0.0 }       &\makecell{ 50$\pm$0.0 }&\makecell{50$\pm$0.0}&\makecell{50$\pm$0.0}&   \makecell{ 50$\pm$0.0 }    \\ 
        pipesworld-tankage      &\makecell{42}&\makecell{{45$\pm$0.6}}&\makecell{43}&\makecell{43}&\makecell{ 43$\pm$0.6 }         &\makecell{ 44$\pm$1.2 }&\makecell{48$\pm$1.0}&\makecell{47$\pm$1.0}&   \makecell{ 48$\pm$1.0 }   \\ 
        psr-small               &\makecell{50}&\makecell{50$\pm$0.0}&\makecell{50}&\makecell{50}&\makecell{ 50$\pm$0.0 }           &\makecell{ 49$\pm$0.0 }&\makecell{49$\pm$0.0}&\makecell{49$\pm$0.0}&    \makecell{ 50$\pm$0.0 }   \\ 
        quantum-layout          &\makecell{20}&\makecell{20$\pm$0.0}&\makecell{20}&\makecell{20}&\makecell{ 20$\pm$0.0 }           &\makecell{ 20$\pm$0.0 }&\makecell{20$\pm$0.0}&\makecell{20$\pm$0.0}&  \makecell{ 20$\pm$0.0 }    \\ 
        recharging-robots       &\makecell{11}&\makecell{{14$\pm$1.0}}&\makecell{13}&\makecell{13}&\makecell{ 14$\pm$1.0 }         &\makecell{ 14$\pm$0.0 }&\makecell{14$\pm$0.0}&\makecell{{14$\pm$0.0}}&    \makecell{ 14$\pm$0.6 }   \\ 
        ricochet-robots         &\makecell{{20}}&\makecell{18$\pm$0.0}&\makecell{14}&\makecell{18}&\makecell{ 20$\pm$0.0 }         &\makecell{ 19$\pm$0.0 }&\makecell{20$\pm$0.0}&\makecell{{20$\pm$0.0}}&   \makecell{ 20$\pm$0.0 }    \\ 
        rovers                  &\makecell{39}& \makecell{40$\pm$0.6}&\makecell{{40}}&\makecell{{40}}&\makecell{ 40$\pm$0.6 }      &\makecell{ 39$\pm$1.7 }&\makecell{{39$\pm$1.0}}&\makecell{40$\pm$0.0}&    \makecell{ 39$\pm$1.0 }   \\ 
        rubiks-cube             &\makecell{6}& \makecell{5$\pm$0.6}&\makecell{{20}}&\makecell{{20}}&\makecell{ 5$\pm$0.0 }         &\makecell{ 5$\pm$0.0 }&\makecell{{5$\pm$0.0}}&\makecell{5$\pm$0.0}&   \makecell{ 5$\pm$0.0 }   \\ 
        satellite               &\makecell{32}& \makecell{34$\pm$0.0}&\makecell{{36}}&\makecell{{36}}&\makecell{ 32$\pm$1.0 }      & \makecell{ 34$\pm$0.0 }&\makecell{{34$\pm$0.6}}&\makecell{33$\pm$0.6}&    \makecell{ 33$\pm$1.2 }   \\ 
        scanalyzer-sat11-strips &\makecell{20}& \makecell{20$\pm$0.6}&\makecell{20}&\makecell{20}&\makecell{ 20$\pm$0.0 }          &\makecell{ 20$\pm$0.0 }&\makecell{19$\pm$0.0}&\makecell{19$\pm$0.6}&    \makecell{ 19$\pm$0.0 }   \\ 
        schedule                &\makecell{{150}}&  \makecell{150$\pm$0.0}&\makecell{{150}}&\makecell{{150}}&\makecell{ 149$\pm$0.6 }   &\makecell{ 150$\pm$0.0 }&\makecell{{150$\pm$0.0}}&\makecell{150$\pm$0.0}&   \makecell{ 150$\pm$0.0 } \\ 
        settlers-sat18-adl      &\makecell{7}& \makecell{12$\pm$0.6}&\makecell{17}&\makecell{{18}}&\makecell{ 11$\pm$0.6 }         &\makecell{ 18$\pm$1.7 }&\makecell{{19$\pm$0.6}}&\makecell{19$\pm$0.6}&    \makecell{ 19$\pm$0.6 }   \\ 
        slitherlink             &\makecell{{6}}& \makecell{{5$\pm$0.6}}&\makecell{0}&\makecell{0}&\makecell{ 6$\pm$0.6 }           &\makecell{ 5$\pm$0.6 }&\makecell{7$\pm$0.0}&\makecell{{5$\pm$0.6}}&   \makecell{ 7$\pm$0.0 }    \\ 
        snake-sat18-strips      &\makecell{17}& \makecell{{20$\pm$0.0}}&\makecell{5}&\makecell{14}&\makecell{ 20$\pm$0.0 }         &\makecell{ 19$\pm$0.6 }&\makecell{19$\pm$0.0}&\makecell{{18$\pm$0.0}}&  \makecell{ 19$\pm$0.0 }    \\ 
        sokoban-sat11-strips    &\makecell{18}& \makecell{15$\pm$0.0}&\makecell{19}&\makecell{19}&\makecell{ 16$\pm$0.6 }          &\makecell{ 12$\pm$2.1 }&\makecell{{15$\pm$0.6}}&\makecell{13$\pm$1.0}&   \makecell{ 16$\pm$0.6 }   \\ 
        spider-sat18-strips     &\makecell{16}& \makecell{17$\pm$1.2}&\makecell{16}&\makecell{16}&\makecell{ 18$\pm$0.0 }          &\makecell{ 16$\pm$0.6 }&\makecell{19$\pm$0.6}&\makecell{{19$\pm$1.5}}&  \makecell{ 19$\pm$0.0 }     \\ 
        storage                 &\makecell{30}& \makecell{{30$\pm$0.0}}&\makecell{20}&\makecell{25}&\makecell{ 30$\pm$0.0 }        &\makecell{ 30$\pm$0.0 }&\makecell{30$\pm$0.6}&\makecell{30$\pm$0.5}&   \makecell{ 30$\pm$0.0 }   \\ 
        termes-sat18-strips     &\makecell{10}& \makecell{5$\pm$2.0}&\makecell{{16}}&\makecell{14}&\makecell{ 10$\pm$0.6 }         &\makecell{ 9$\pm$0.6 }&\makecell{10$\pm$0.6}&\makecell{10$\pm$0.6}&   \makecell{ 10$\pm$0.0 }  \\ 
        tetris-sat14-strips     &\makecell{17}& \makecell{{20$\pm$0.0}}&\makecell{16}&\makecell{17}&\makecell{ 20$\pm$0.0 }        &\makecell{ 20$\pm$0.0 }&\makecell{20$\pm$0.0}&\makecell{{20$\pm$0.0}}&  \makecell{ 20$\pm$0.0 }     \\ 
        thoughtful-sat14-strips &\makecell{{20}}& \makecell{{20$\pm$0.0}}&\makecell{15}&\makecell{19}&\makecell{ 20$\pm$0.0 }      &\makecell{ 20$\pm$0.0 }&\makecell{20$\pm$0.0}&\makecell{{20$\pm$0.0}}&   \makecell{ 20$\pm$0.0 }    \\ 
        tidybot-sat11-strips    &\makecell{18}& \makecell{{20$\pm$0.0}}&\makecell{17}&\makecell{{20}}&\makecell{ 20$\pm$0.0 }      &\makecell{ 20$\pm$0.0 }&\makecell{{19$\pm$0.0}}&\makecell{{19$\pm$0.0}}&    \makecell{ 19$\pm$0.0 }  \\ 
        tpp                     &\makecell{30}& \makecell{30$\pm$0.0}&\makecell{30}&\makecell{30}&\makecell{ 30$\pm$0.0 }          &\makecell{ 30$\pm$0.0 }&\makecell{30$\pm$0.0}&\makecell{30$\pm$0.0}&  \makecell{ 30$\pm$0.0 }  \\
        transport-sat14-strips  &\makecell{{20}}& \makecell{{20$\pm$0.0}}&\makecell{17}&\makecell{18}&\makecell{ 20$\pm$0.0 }      &\makecell{ 20$\pm$0.0 }&\makecell{20$\pm$0.0}&\makecell{{20$\pm$0.0}}&   \makecell{ 20$\pm$0.0 }   \\
        trucks-strips           &\makecell{19}& \makecell{13$\pm$1.0}&\makecell{18}&\makecell{20}&\makecell{ 18$\pm$0.0 }          &\makecell{ 9$\pm$1.2 }&\makecell{{9$\pm$0.6}}&\makecell{9$\pm$0.6}&  \makecell{ 18$\pm$0.0 }   \\ 
        visitall-sat14-strips   &\makecell{20}& \makecell{20$\pm$0.0}&\makecell{20}&\makecell{20}&\makecell{ 20$\pm$0.0 }          &\makecell{ 20$\pm$0.0 }&\makecell{20$\pm$0.0}&\makecell{20$\pm$0.0}&  \makecell{ 20$\pm$0.0 }   \\ 
        woodworking-sat11-strips &\makecell{{20}}& \makecell{13$\pm$1.5}&\makecell{{20}}&\makecell{{20}}&\makecell{ 20$\pm$0.0 }   & \makecell{ 20$\pm$0.0 }&\makecell{{20$\pm$0.0}}&\makecell{{20$\pm$0.0}}&    \makecell{ 20$\pm$0.0 }   \\ 
        zenotravel              &\makecell{20}& \makecell{20$\pm$0.0}&\makecell{20}&\makecell{20}&\makecell{ 20$\pm$0.0 }          &\makecell{ 20$\pm$0.0 }&\makecell{20$\pm$0.0}&\makecell{20$\pm$0.0}&   \makecell{ 20$\pm$0.0 }   \\ \hline
        \rule{0pt}{2.2ex}\textbf{Coverage (1831)} &\makecell{1607}& \makecell{1611$\pm$3.5}& \makecell{1535}& \makecell{1591}& \makecell{1641$\pm$0.6}                    &\makecell{ 1608$\pm$2.7 }&\makecell{1621$\pm$3.2}& \makecell{1616$\pm$2.1} & \makecell{1655$\pm$1.5}    \\ \hline
         \rule{0pt}{2.2ex}\textbf{\% Score (100\%)} &\makecell{83.56\%}& \makecell{83.83\%$\pm$0.17}& \makecell{79.07\%}& \makecell{82.92\%}& \makecell{86.23\%$\pm$0.06}       &\makecell{ 83.75\%$\pm$0.09 }&\makecell{84.59\%$\pm$0.25}& \makecell{84.22\%$\pm$0.13}& \makecell{86.99\%$\pm$0.08}   \\\hline
         \rule{0pt}{2.2ex}\textbf{Agile score}    &\makecell{1200.9}& \makecell{1233.7$\pm$0.24}& \makecell{1192.3}& \makecell{1206.4}&\makecell{1173.3$\pm$3.5}  & \makecell{ 1215.9$\pm$4.0 }&\makecell{1229.4$\pm$3.4} &\makecell{1236.0$\pm$1.9} &\makecell{1232.6$\pm$2.9}  \\ \hline
    \end{tabular}
    }
    \caption{Comparative performance analysis across the full set of benchmark domains. \textit{\% score} is the average of the \% of instances solved in each domain. Values for solvers with randomized components represent the mean and include the standard deviation across 3 measurements, using seeds from 0 to 2.}
    \label{tab:comparative_performance_large2}
\end{table*}

\end{document}

%% file: macros.tex
\usepackage{amssymb}

\newcommand{\tuple}[1]{\ensuremath{\langle #1 \rangle}} 	% tuple
\newcommand{\set}[1]{\{#1\}}                      	% set
\newcommand{\card}[1]{|{#1}|}                     	% cardinality of a set

\newcommand{\defterm}[1]{\textbf{\textit{#1}}}

\newcommand{\fancy}[1]{\ensuremath{\mathcal{#1}}}

\newcommand{\qq}[1]{``#1"}		  % quotes
\newcommand{\nspar}[1]{\vspace{2mm}\noindent \textbf{#1. }}
\newtheorem{theorem}{Theorem}

\newtheorem{lemma}{Lemma}
\newtheorem{claim}{Claim}
\newtheorem{corollary}{Corollary}

\makeatletter
\let\r@eqnnum\@eqnnum
\input{leqno.clo}
\let\l@eqnnum\@eqnnum
\newcommand{\leqnos}{\let\@eqnnum\l@eqnnum}
\newcommand{\reqnos}{\let\@eqnnum\r@eqnnum}
\reqnos
\makeatother

\newcommand{\E}{\modecal{E}} 
 
 \newcommand{\J}{\modecal{J}}
 \renewcommand{\L}{\modecal{L}}
\newcommand{\M}{\modecal{M}} 
\renewcommand{\O}{\modecal{O}} \renewcommand{\P}{\modecal{P}}
 \newcommand{\T}{\modecal{T}}
 
 \newcommand{\X}{\modecal{X}}
\newcommand{\Y}{\modecal{Y}} \newcommand{\Z}{\modecal{Z}}

\newcommand{\modecal}[1]{\ensuremath{\mathcal{#1}}}

\newcounter{rgcounter}
\refstepcounter{rgcounter}

\usepackage{xspace}

  \newcommand\wrapfill{\par
   \ifx\parshape\WF@fudgeparshape
   \nobreak
   \vskip-\baselineskip
   \vskip\c@WF@wrappedlines\baselineskip
   \allowbreak
   \WFclear
   \fi
 }

 \usepackage{pifont}% http://ctan.org/pkg/pifont
\usepackage[
  textsize=scriptsize,
  colorinlistoftodos,
  ]{todonotes}
\let\todonote\todo
\renewcommand{\todo}[2]{\todonote[color=red!50]{(#1): #2}}